\documentclass{article}
\usepackage[preprint]{neurips_2026}

\usepackage[utf8]{inputenc}
\usepackage{amsmath,amsthm}
\usepackage[libertine]{newtxmath}
\usepackage{array}
\usepackage{subcaption}
\usepackage{graphicx}
\usepackage{xcolor}
\usepackage{algorithm}
\usepackage{algorithmicx}
\usepackage[noend]{algpseudocode}
\usepackage{booktabs}
\usepackage{colortbl}
\usepackage{arydshln}
\usepackage{hyperref}
\usepackage{makecell}
\usepackage{siunitx}
\usepackage{appendix}
\usepackage{xspace}
\usepackage{enumitem}
\usepackage{rotating}
\usepackage{multirow}
\newcommand{\E}{\mathbb{E}}
\definecolor{pinegreen}{rgb}{0.0, 0.47, 0.44}
\newcommand{\modelname}{\textbf{FedFrozen}\xspace}

\newtheorem{thm}{\textbf{Theorem}}
\newtheorem{col}{\textbf{Corollary}}
\newtheorem{asmp}{\textbf{Assumption}}
\newtheorem{lem}{\textbf{Lemma}}

\newtheorem{prop}{\textbf{Proposition}}
\newtheorem{remark}{\textbf{Remark}}

\DeclareMathOperator*{\argmin}{arg\,min}
\newcommand{\PhiFrozen}{\bar{\Phi}}

\newcommand{\norm}[1]{\left\lVert #1 \right\rVert}
\newcommand{\inner}[2]{\left\langle #1, #2 \right\rangle}

\graphicspath{{./}{./image/}{./figure/}}

\title{FedFrozen: Two-Stage Federated Optimization via Attention Kernel Freezing}
\author{Junye Du\thanks{Equal contribution.} \\
  The University of Hong Kong \\
  \texttt{junyedu@connect.hku.hk}
  \And
  Zhenghao Li\footnotemark[1] \\
  The University of Hong Kong \\
  \texttt{lizh@connect.hku.hk}
  \And
  Yushi Feng \\
  The University of Hong Kong \\
  \texttt{fengys@connect.hku.hk}
  \And
  Long Feng\thanks{Corresponding author.} \\
  The University of Hong Kong \\
  \texttt{lfeng@hku.hk}}
\date{}

\begin{document}
\maketitle

\begin{abstract}
Federated learning with heterogeneous clients remains a significant challenge for deep learning, primarily due to client drift arising from inconsistent local updates. Existing federated optimization methods typically address this issue through objective-level regularization or update-correction mechanisms. Recent studies, however, suggest that Transformer-based architectures may be inherently more robust than conventional models under heterogeneous federated training. Motivated by this observation, we investigate how different parameter components within the attention mechanism influence federated optimization. Specifically, we decompose the attention module into a query/key block, which determines the attention kernel, and a value block, which performs semantic transformation under the induced kernel. Based on this perspective, we propose \modelname, a two-stage federated optimization framework that first performs full-model warm-up training and then freezes the query/key block while continuing to optimize the value block. Under a linear-attention formulation, we show that the warm-up stage can be interpreted as an inexact descent procedure on a regularized kernel-profile objective, while the frozen stage reduces to a restricted value-block optimization problem under a fixed attention kernel. Our analysis further reveals an explicit trade-off that governs the choice of warm-up length. Simulations validate the predicted bias–drift behavior, and real-data experiments demonstrate that \modelname improves both the stability and effectiveness of Transformer models in heterogeneous federated learning.
\end{abstract}

\section{Introduction}
Federated learning (FL) aims to optimize a shared model across decentralized clients while keeping data local. In practice, however, data heterogeneity induces substantial client drift: after multiple local steps, client models may move in directions that are inconsistent with the global objective, which can compromise optimization stability and final accuracy \citep{mcmahan2017communication,li2020convergence,karimireddy2020scaffold}. This issue becomes particularly pronounced for large-scale deep neural networks, where the large number of parameters makes local training more vulnerable to client-specific data heterogeneity.

Existing methods address this issue mainly at the whole-model optimization level. FedProx adds a proximal regularizer to stabilize local training \citep{li2020fedprox}; SCAFFOLD uses control variates to correct client drift \citep{karimireddy2020scaffold}; FedNova uses normalized averaging to correct objective inconsistency arising from heterogeneous local updates \citep{wang2020fednova}; FedDyn dynamically regularizes local objectives to better align local and global optima \citep{acar2021feddyn}; and MOON uses model-level contrastive learning to reduce representation drift under non-IID data \citep{li2021moon}. While these methods improve federated optimization, they mainly operate at the level of the overall training objective and have not specifically investigated how the underlying neural network architecture affects federated training.

This limitation becomes particularly relevant for Transformer-style models \citep{vaswani2017attention}. Compared with conventional feed-forward networks, Transformer architectures are composed of multiple tightly coupled components, including query, key, and value projections in self-attention modules as well as subsequent MLP layers. Recent studies suggest that self-attention-based architectures can be more robust than conventional architectures under heterogeneous federated learning \citep{qu2022rethinking}. However, subsequent work shows that naive federated averaging may negatively affect the self-attention mechanism when client data are heterogeneous \citep{li2022fedtp}. These findings indicate that attention is not merely another layer to be averaged: its internal components may respond differently to local client updates.

Several related works, including block-wise training and parameter-efficient tuning, have considered updating only part of a model, but their focus differs from ours. Personalized FL methods split models into shared and local components \citep{arivazhagan2019federated,collins2021exploiting,oh2022fedbabu}; parameter-efficient tuning methods reduce trainable parameters through adapters or low-rank updates \citep{houlsby2019adapter,hu2022lora}. Recent FL methods further show that updating only part of the network can improve efficiency and optimization behavior. FedPart provides a non-convex convergence analysis for layer-wise partial network updates in federated learning \citep{wang2024fedpart}. SmartFreeze and FedFreeze introduce freezing mechanisms for improving efficiency and training stability, but they do not study the distinct optimization roles of the query/key and value projections \citep{wu2024smartfreeze,wu2026fedfreeze}. In contrast, our analysis instead focuses on the role of different attention components in federated training. Freezing the query/key block fixes the attention kernel and its induced representation geometry, so that later value-block updates take place within a stable representation space.

Motivated by this observation, we propose \modelname, a two-stage federated optimization method for attention-based models. In the first stage, the full attention module is trained for a short warm-up period, allowing the query/key block to learn a useful shared kernel. In the second stage, the query/key block is fixed, and federated training continues by updating and aggregating only the value block. This design reduces the communication cost after warm-up and, more importantly, mitigates the impact of client drift on global model training under heterogeneous data.

Our contributions are summarized as follows. First, we propose \modelname, a two-stage federated optimization framework. Second, to the best of our knowledge, this is the first work to analyze federated linear attention by explicitly decomposing it into a query/key-induced kernel block and a value-based active block. This attention-specific decomposition distinguishes our approach from generic partial-update or layer-freezing methods. Third, we provide a unified theoretical analysis of the two stages. The warm-up phase is shown to perform inexact descent on a regularized kernel-profile objective, whereas the frozen phase becomes a smooth and strongly convex restricted value-block optimization problem under our settings. Finally, relevant simulation results support our theoretical analysis, and real-data experiments demonstrate the effectiveness of \modelname on Transformer models under heterogeneous federated training.

\textbf{Notation:} Let $\| \cdot \|_F$ and $\| \cdot \|_{\mathrm{op}}$ denote the Frobenius and operator norms, respectively. For matrices $A$ and $B$, their Euclidean inner product is defined as $\langle A, B \rangle = \mathrm{tr}(A^\top B)$. For a differentiable matrix-valued function $\mathcal{T}$, let $D\mathcal{T}(X)[U]$ denote its Fr\'echet derivative at $X$ applied to a perturbation direction $U$. Its induced operator norm is defined as $ \|D\mathcal{T}(X)\|_{\mathrm{op}} = \sup_{\|U\|_F=1} \|D\mathcal{T}(X)[U]\|_F.$ For a twice-differentiable scalar-valued function $\mathcal{T}$ with matrix input $X$, we let $\nabla \mathcal{T}(X)$ and $\nabla^2 \mathcal{T}(X)$ denote its gradient and Hessian.

\section{Methodology}
\subsection{Problem Formulation}
We consider a federated learning system with $K$ heterogeneous clients. Let $\mathcal{D}_k$ denote the local data distribution of client $k\in\{1,\dots,K\}$. The standard federated objective is to minimize the weighted global expected risk over parameter space $\mathcal{W}$:
\begin{equation}
\label{eq:global_empirical_risk}
    \min_{w\in\mathcal{W}} \left \{ f(w) := \sum_{k=1}^K p_k F_k(w) \right \},
\end{equation}
where $\{p_k\}_{k=1}^K$ are nonnegative weights summing to $1$, and $F_k(w) := \E_{(x,y)\sim\mathcal{D}_k} \bigl[\ell(\mathcal{M}(x;w),y)\bigr]$ is the local objective of client $k$ for model $\mathcal{M}$ and loss function $\ell$.

While existing federated optimization methods typically treat the model parameter as a uniform whole—for instance, by applying a single regularization term across all weights equally—this approach fails to capture the distinct component roles within Transformer-style architectures. Specifically, attention models can be particularly sensitive to client drift because local updates implicitly disrupt the shared attention patterns. To mitigate this instability, we focus on the distinct roles of the internal components within the self-attention mechanism.

Let $H(x)\in\mathbb{R}^{n\times d}$ denote the input embedding for sample $x$.  To facilitate a rigorous theoretical analysis, we consider a linear-attention model surrogate:
\begin{equation}
\label{eq:attention_module}
    O(H;W^Q,W^K,W^V)=\varphi(HW^Q)\varphi(HW^K)^\top HW^V,
\end{equation}
where $W^Q\in\mathbb{R}^{d\times d_k}$, $W^K\in\mathbb{R}^{d\times d_k}$ and $W^V\in\mathbb{R}^{d\times d_v}$ are the trainable parameters, and $\varphi$ is a positive feature map applied row-wise. We focus on this linear attention model primarily to provide a clear theoretical analysis; algorithmically, however, our proposed \modelname framework can be generalized to standard Transformer models by directly designating the corresponding frozen and active blocks. This specific architecture exposes a natural `division of labor'. The query/key matrices define the attention kernel $A_\Phi(H) = \varphi(HW^Q)\varphi(HW^K)^\top$, while the value block acts as a linear map $V(H;W^V) = HW^V$. We therefore formally decompose the parameter space into a kernel-induced block and an active value block:
\begin{equation*}
    \Phi = (W^Q,W^K),
    \qquad
    \theta = W^V.
\end{equation*}
This decomposition highlights an asymmetry that is absent in standard
whole-model FL analysis. Client drift in $\Phi$ alters the attention kernel and changes how token representations interact. In contrast, once $\Phi$ is fixed, the value block $\theta$ is restricted to learning a linear mapping upon fixed attention features. Consequently, rather than directly optimizing the monolithic objective $f(w)$, our analysis and algorithm design will turn to the decoupled local objective:
\begin{equation}
    F_k(\Phi,\theta) = \mathbb{E}_{(x,y)\sim\mathcal{D}_k} \bigl[ \ell\bigl(O(H(x);\Phi,\theta),y\bigr) \bigr].
\end{equation}
This allows us to explore how selectively freezing the kernel could control heterogeneous client drift by optimizing the following decoupled global objective:
\begin{equation}
    \min_{\Phi,\theta} \left \{ f(\Phi,\theta)  = \sum_{k=1}^K p_k F_k(\Phi,\theta) \right \},  \  \ \sum_{k=1}^K p_k = 1, \ \ p_k \ge 0.
\end{equation}

\subsection{Algorithm Design}
The component-wise decomposition introduced above suggests that directly applying standard FedAvg \citep{mcmahan2017communication} to all attention parameters indiscriminately leaves the model vulnerable to client drift. To address this asymmetry, we propose \modelname, a two-stage federated optimization strategy that stabilizes training by separating kernel adaptation from value-block optimization.

Our algorithm proceeds in two phases. In the first phase, consisting of $T_{\rm warm}$ communication rounds, all clients jointly update both the kernel block $\Phi$ and the active value block $\theta$. The primary role of this full-model warm-up is to allow the attention mechanism to learn a shared, representative kernel before it is fixed. Consequently, a sufficiently long warm-up period helps reduce the future optimization bias that would otherwise arise from premature freezing.

Following the warm-up phase, the server freezes the kernel parameters at $\PhiFrozen := \Phi^{T_{\rm warm}}$ and broadcasts this fixed kernel to all clients. For the remaining training rounds (Phase 2), local training is restricted exclusively to the value block $\theta$ via gradient descent, with $\PhiFrozen$ held constant. While Phase 2 can no longer alter the underlying representation geometry, it efficiently reduces the optimization residual of the active block within a stable, globally shared representation space.

To further stabilize the training and ensure that the restricted optimization problem in Phase 2 is well-conditioned, we introduce an $\ell_2$-regularization penalty specifically on the active value block. The regularized local objective is thus defined as:
\begin{equation}
\label{eq:def_local_regularized}
    F_k^\lambda(\Phi,\theta) := F_k(\Phi,\theta)+\frac{\lambda}{2}\|\theta\|_F^2, \quad \lambda>0.
\end{equation}
Accordingly, the global regularized objective becomes:
\begin{equation}
\label{eq:def_global_regularized}
    f_\lambda(\Phi,\theta) := \sum_{k=1}^K p_k F_k^\lambda(\Phi,\theta) = f(\Phi,\theta)+\frac{\lambda}{2}\|\theta\|_F^2, \quad p_k \ge 0, \ \sum_{k=1}^K p_k = 1.
\end{equation}
Finally, while our subsequent theoretical analysis assumes full client participation for clarity, \modelname naturally accommodates partial participation. Client subsampling can be incorporated by adding standard sampling-variance terms to the analysis, which has no structural impact on our core kernel-freezing mechanism.

\begin{algorithm}[htbp!]
\caption{\modelname: Two-Stage Federated Optimization with $\ell_2$-regularization}
\label{alg:fedfrozen}
\begin{algorithmic}[1]
\Require Client weights $\{p_k\}_{k=1}^K$, $\ell_2$ regularization parameter $\lambda>0$, warm-up rounds $T_{\rm warm}$, total rounds $T$, local step size $\eta$, local steps $E$
\State \textbf{Initialize:} server parameters $(\Phi^0,\theta^0)$
\For{$t = 0, \dots, T_{\rm warm}-1$} \Comment{\textcolor{pinegreen}{Phase 1: full-model warm-up}}
    \State Server broadcasts $(\Phi^t,\theta^t)$ to clients
    \For{each client $k$ \textbf{in parallel}}
        \State Local init: $\Phi_{k,0}^t \gets \Phi^t$, $\theta_{k,0}^t \gets \theta^t$
        \For{$j = 0, \dots, E-1$}
            \State $(\Phi_{k,j+1}^t,\theta_{k,j+1}^t) \gets (\Phi_{k,j}^t,\theta_{k,j}^t) - \eta\nabla F_k^\lambda(\Phi_{k,j}^t,\theta_{k,j}^t)$
        \EndFor
    \EndFor
    \State Server aggregates: $(\Phi^{t+1},\theta^{t+1}) \gets \sum_{k=1}^K p_k (\Phi_{k,E}^t,\theta_{k,E}^t)$
\EndFor
\State Freeze kernel block: $\PhiFrozen \gets \Phi^{T_{\rm warm}}$
\For{$t = T_{\rm warm}, \dots, T-1$} \Comment{\textcolor{pinegreen}{Phase 2: frozen-kernel optimization}}
    \State Server broadcasts active block $\theta^t$ with frozen kernel $\PhiFrozen$
    \For{each client $k$ \textbf{in parallel}}
        \State Local init: $\theta_{k,0}^t \gets \theta^t$
        \For{$j = 0, \dots, E-1$}
            \State $\theta_{k,j+1}^t \gets \theta_{k,j}^t - \eta \nabla_\theta F_k^\lambda(\PhiFrozen,\theta_{k,j}^t)$
        \EndFor
    \EndFor
    \State Server aggregates active block: $\theta^{t+1} \gets \sum_{k=1}^K p_k \theta_{k,E}^t$
\EndFor
\State \textbf{Output:} final federated model $(\PhiFrozen,\theta^{T})$
\end{algorithmic}
\end{algorithm}

\section{Theoretical Analysis}
\label{sec:theory}
In this section, we provide a theoretical analysis of why freezing the query/key block after the warm-up phase could stabilize federated learning optimization. Our analysis proceeds in three steps. First, we introduce a profile objective
$h_\lambda(\Phi)$, which measures the best
regularized objective attainable after fixing the kernel $\Phi$. We show that the
warm-up stage implements an inexact descent procedure on this profile
objective. Second, after the kernel is frozen, we show
that the remaining value-block optimization becomes a smooth and strongly
convex federated problem under the linear-attention model, yielding linear convergence to a neighborhood. Third, we combine the
two phase-wise guarantees into an end-to-end error decomposition that
identifies the trade-off governing the choice of $T_{\rm warm}$.

\subsection{Phase 1: Warm-up as Inexact Descent}
We first analyze the warm-up phase of \modelname. Since the full
attention objective is nonconvex in the query/key block, our goal is not to
establish global convergence of Phase~1.  Instead, the relevant question is
whether warm-up moves the model toward a kernel that will incur a smaller
bias before the query/key block is frozen.  To formalize this effect, we
introduce the regularized kernel-profile objective
\begin{equation}
\label{eq:def_main_notation}
    h_\lambda(\Phi) := \min_\theta f_\lambda(\Phi,\theta),
\end{equation}
which measures the best attainable objective under a fixed kernel $\Phi$.  A smaller value of $h_\lambda(\Phi)$ therefore indicates that the
frozen kernel $\Phi$ supports a better value-block solution.

Based on this profile objective, we define the regularized freezing bias at
round $t$ as
\begin{equation}
\label{eq:def_bias_residual}
    B_{t}^{\lambda}
    :=
    h_\lambda(\Phi^t)-f_\lambda^\star \quad \text{where}\quad f_\lambda^\star := \min_{\Phi,\theta} f_\lambda(\Phi,\theta).
\end{equation}
Before we present the convergence property of Phase~1, we first propose the following assumptions.

\begin{asmp}[Bounded linear-attention assumptions] \
\label{asmp:attention_assumption}
\begin{enumerate}[label=(A1.\arabic*),ref=A1.\arabic*,leftmargin=3em]
    \item\label{asmp:1}
    The Frobenius norm of the input embedding is bounded: $\norm{H}_F \le R$;

    \item\label{asmp:2}
    The loss gradient with respect to the attention output is bounded:
    $\norm{\nabla_O \ell(O,y)}_F \le G_O$;

    \item\label{asmp:3}
    All iterations of the value block satisfy $\norm{\theta}_F \le B_\theta$;

    \item\label{asmp:4}
    The row-wise feature map $\varphi$ is continuously differentiable, bounded, and has a bounded Lipschitz-continuous Jacobian. That is, for any $u,v$:
    \[
        \norm{\varphi(u)}_F \le B_\varphi,
        \qquad
        \norm{D\varphi(u)}_{\rm op} \le J_\varphi,
        \qquad
        \norm{D\varphi(u)-D\varphi(v)}_{\rm op}
        \le
        M_{\varphi} \norm{u-v}_F,
    \]
    where $D$ is the Fr\'echet derivative operator;

    \item\label{asmp:5}
    For every label $y$, the loss function $\ell(O,y)$ is convex and
    $L_O$-smooth with respect to $O$:
    \[
        \ell(O',y)\le \ell(O,y)+\inner{\nabla_O \ell(O,y)}{O'-O}
        +\frac{L_O}{2}\norm{O'-O}_F^2.
    \]
\end{enumerate}
\end{asmp}
Assumption~\ref{asmp:attention_assumption} states common regularity
conditions for the linear-attention model. Conditions
(\ref{asmp:1})--(\ref{asmp:3}) are common boundedness assumptions on the input
embedding $H$, output gradients $\nabla_O\ell$, and value-block $\theta$. Condition~(\ref{asmp:4})
requires the feature map to be continuously differentiable, bounded, and to
have a bounded Lipschitz Jacobian, which is satisfied by commonly used smooth
bounded feature maps, such as sigmoid and softmax-type maps. Condition~(\ref{asmp:5})
assumes that the loss function is convex and smooth with respect to the attention
output $O$. This condition is satisfied by the mean squared error loss and also
covers standard cross-entropy loss for classification with a fixed linear
prediction head.

\begin{asmp}
\label{asmp:global_dissim}
There exist constants $M\ge 0$ and $B\ge 1$ such that
\begin{equation}
\label{eq:global_dissim}
    \sum_{k=1}^K p_k
    \norm{\nabla F_{k}(\Phi,\theta)}_F^2
    \le
    M^2
    +
    B^2
    \norm{\nabla f(\Phi,\theta)}_F^2
\end{equation}
\end{asmp}
Assumption~\ref{asmp:global_dissim} is the bounded gradient dissimilarity
condition used in SCAFFOLD~\citep{karimireddy2020scaffold}, which measures
the difference between aggregated local client gradients and the global gradient under
heterogeneous data. We next state Theorem~\ref{thm1} for the warm-up phase.

\begin{thm}
    \label{thm1}
    For Phase 1 of Algorithm~\ref{alg:fedfrozen}, assume Assumptions~\ref{asmp:attention_assumption} and~\ref{asmp:global_dissim} hold and that there exists a unique $\theta_\lambda^\star(\Phi) := \argmin_\theta f_\lambda(\Phi,\theta)$ for any kernel $\Phi$. Let $e_t := \theta^t-\theta_\lambda^\star(\Phi^t)$. If $\eta < c \lambda/ E$ for some constant $c>0$ and $\lambda \ge C/(2B_\theta^2)$ for some constant $C>0$, then every warm-up round $t \in \{0, 1, \dots, T_{\rm warm}-1\}$ satisfies
    \begin{equation}
        \label{eq:profile_descent_la}
    h_\lambda(\Phi^{t+1})
    \le
    h_\lambda(\Phi^t)
    -
    \frac{3\eta E}{8}
    \norm{\nabla h_\lambda(\Phi^t)}_F^2
    +
    C_e\eta E
    \norm{e_t}_F^2 + \frac{(M^2+C(B^2-1)\lambda)\eta E}{16B^2},
    \end{equation}
    for some constant $C_e>0$. Furthermore, the freezing bias at time $t$ satisfies
    \begin{equation}
        \label{eq:future_bias_reduction}
    B_{t}^{\lambda}
    \le
    B_{0}^{\lambda}
    -
    \frac{3\eta E}{8}
    \sum_{s=0}^{t-1}
    \norm{\nabla h_\lambda(\Phi^s)}_F^2
    +
    C_e\eta E
    \sum_{s=0}^{t-1}
        \norm{e_s}_F^2
    + \frac{(M^2+C(B^2-1)\lambda)\eta E}{16B^2} t.
    \end{equation}
\end{thm}

Theorem~\ref{thm1} should be interpreted as an inexact descent guarantee for
the kernel profile. The first $\|e_t\|_F^2$ measures
how far the current value block is from the optimal value block associated
with the current kernel, while the second term captures heterogeneity-induced client drift. In the homogeneous case, where
local gradients are equal to the global gradient, one can take $M=0$ and
$B=1$, and thus the second error term vanishes. As such, Phase 1 should not be regarded as guaranteed monotone improvement for arbitrary warm-up length; rather, it provides useful descent when the profile decrease outweighs tracking and heterogeneity errors.

Additionally assume that for every iterated kernel $\Phi^t$, the minimizer $\theta^\star(\Phi^t) := \argmin_\theta f(\Phi^t,\theta)$ exists and satisfies $\|\theta^\star(\Phi^t)\|_F\le B_\theta$. Then we have the following Corollary~\ref{col1}.

\begin{col}
\label{col1}
    Under the conditions of Theorem~\ref{thm1}, let $h(\Phi) = h_0(\Phi)$ denote the unregularized profile for kernel $\Phi$, and let $f^\star = f^\star_0$ be the original global optimum. Then the global freezing bias $B_t = h(\Phi^t) - f^\star$ for $t \in \{1, 2, \dots, T_{\rm warm}\}$ satisfies
    \begin{equation}
    B_t \le B_{0} - \frac{3\eta E}{8} \sum_{s=0}^{t-1} \norm{\nabla h_\lambda(\Phi^s)}_F^2 + C_e\eta E \sum_{s=0}^{t-1} \norm{e_s}_F^2 + \frac{(M^2+C(B^2 -1)\lambda)\eta E}{16B^2} t  + \frac\lambda2 B_\theta^2,
    \end{equation}
\end{col}
Corollary~\ref{col1} gives the corresponding freezing-bias bound for the
original objective. Compared with the regularized guarantee in
Theorem~\ref{thm1}, it incurs only the additional regularization error
$\lambda B_\theta^2/2$.

\subsection{Phase 2: Frozen-kernel optimization}
We now turn to the second phase of \modelname. Once the kernel block is frozen at $\PhiFrozen = \Phi^{T_{\rm warm}}$, all
clients share the same attention kernel and only update the value block
$\theta$. Accordingly, we define the restricted local and global objective as
\begin{equation*}
    g_{\PhiFrozen,k}^\lambda(\theta) = F_k^\lambda(\PhiFrozen,\theta),
    \qquad
    g_{\PhiFrozen}^\lambda(\theta) = \sum_{k=1}^K p_k g_{\PhiFrozen,k}^\lambda(\theta).
\end{equation*}
For the analysis in this subsection, we re-index Phase~2 rounds by
$s=0,1,\dots$ and write $\theta^0$ for the value block at the freezing round.
The following theorem gives a theoretical guarantee for this frozen-kernel objective.
\begin{thm}
\label{thm:restricted_descent_relative}
Under Assumption~\ref{asmp:attention_assumption}, if $\eta \le c \sqrt{\lambda}/E$ for some constant $c > 0$, then every Phase~2 round satisfies
\begin{equation}
         g_{\PhiFrozen}^\lambda(\theta^{s+1}) \le g_{\PhiFrozen}^\lambda(\theta^{s}) - \frac{3}{8}  \eta E \norm{\nabla g_{\PhiFrozen}^\lambda(\theta^s) }_F^2 + \frac{C^2\eta E}{2},
\end{equation}
for some constant $C>0$. Consequently,
\begin{equation}
        \frac{1}{S}\sum_{s=0}^{S-1}\norm{\nabla g_{\PhiFrozen}^\lambda (\theta^s)}_F^2 \le \frac{8(g_{\PhiFrozen}^\lambda(\theta^{0}) - h_\lambda(\PhiFrozen))}{3 \eta E S} + \frac{4C^2}{3}
\end{equation}
\end{thm}

Theorem~\ref{thm:restricted_descent_relative} provides an optimization
guarantee for one round of Phase~2 update under a fixed kernel. Since the frozen-kernel objective is strongly convex with respect to the value block under the conditions of Theorem~\ref{thm:restricted_descent_relative}, we have the following Corollary~\ref{col:restricted_linear}.

\begin{col}
    \label{col:restricted_linear}
    Assume Assumption~\ref{asmp:attention_assumption} holds. Let $g_{\PhiFrozen}(\theta) =  g_{\PhiFrozen}^0(\theta)$ and $h(\Phi) = h_0(\Phi)$, and let $S = T - T_{\rm warm}$ be the total number of Phase~2 rounds. Then the following inequality holds:
  \begin{equation}
    \begin{aligned}
        g_{\PhiFrozen}(\theta^S) - h(\PhiFrozen) &\le \left(1-\frac{3}{4}\lambda_0 \eta E \right)^S \left( g_{\PhiFrozen}^{\lambda_0}(\theta^0) - h_{\lambda_0}(\PhiFrozen) \right) + C B_\theta
    \end{aligned}
    \end{equation}
    for $\eta \le \min\{4/(3\lambda_0 E), c\sqrt{\lambda_0}/E \}$, where $\lambda_0 := (6C)/(7B_\theta)$ and $C,c$ are certain positive constants.
\end{col}
Corollary~\ref{col:restricted_linear} shows that the original Phase~2 objective converges to a neighborhood of radius $C B_\theta$. Notably, this guarantee does not rely on the heterogeneity condition in Assumption~\ref{asmp:global_dissim},  indicating that, under the
linear-attention model and Assumption~\ref{asmp:attention_assumption}, freezing the query/key block itself provides a
structural mechanism for controlling heterogeneity.

\subsection{End-to-end Decomposition and Warm-up Trade-off}
We now combine the two phase-wise guarantees to characterize how the warm-up length should be chosen. Instead of directly setting $T_{\rm warm}$,  we now consider an arbitrary candidate freezing round $\tau\in\{0,\dots,T-1\}$ so that the corresponding freezing bias and $\theta$ block residual are defined by:
\begin{equation}
\label{eq:def_twarm_bias_residual}
    B_\tau^\lambda := h_\lambda(\Phi^\tau)-f_\lambda^\star,
    \qquad
    R_\tau^\lambda := f_\lambda(\Phi^\tau,\theta^\tau)-h_\lambda(\Phi^\tau).
\end{equation}
Let $\theta_{\tau,S}$ denote the value block obtained
after running $S$ rounds of Phase~2 initialized at $\theta^\tau$ with the
kernel $\Phi^\tau$ frozen.  We write the final parameter as
\begin{equation*}
        w_{\tau,T} := (\Phi^\tau,\theta_{\tau,T-\tau}).
\end{equation*}
The following proposition gives an end-to-end error decomposition for  candidate
freezing round $\tau$.
\begin{prop}
\label{prop:e2e_certificate}
Assume the conditions of Theorem~\ref{thm1} and
Corollary~\ref{col:restricted_linear}. Let $\rho_\lambda := 1-\frac{3}{4}\lambda\eta E$, and further assume $(\Phi^\star, \theta^\star) := \argmin_{(\Phi, \theta)} f(\Phi,\theta)$ exists and $\norm{\theta^\star}_F \le B_\theta$. Then, for every candidate freezing round $\tau\in\{0,\dots,T-1\}$,
\begin{equation}
    f(w_{\tau,T})-f^\star \le B_\tau^{\lambda_0} +
    \rho_{\lambda_0}^{T-\tau}R_\tau^{\lambda_0} + C B_\theta,
\end{equation}
where $\lambda_0 := (6C)/(7B_\theta)$ and $C>0$ is a positive constant defined in Corollary~\ref{col:restricted_linear}.
\end{prop}

Proposition~\ref{prop:e2e_certificate} decomposes the final error into freezing and residual bias. Then the oracle choice of the warm-up length is obtained
by minimizing only the following term:
\begin{equation}
    T_{\rm warm}^{\rm or}
    \in
    \argmin_{\tau\in\{0,\dots,T-1\}}
    \left\{
        B_\tau^{\lambda_0}
        +
        \rho_{\lambda_0}^{T-\tau}R_\tau^{\lambda_0}
    \right\}.
\end{equation}
This decomposition highlights a trade-off in choosing the warm-up length
$\tau$ under our \modelname framework. If $\tau$ is too small, the query/key block $\Phi$ is frozen before
learning a reliable attention kernel, leading to a large freezing bias
$B_\tau^\lambda$. However, when client heterogeneity is strong, an overly large
$\tau$ can accumulate drift and tracking errors during the inexact warm-up
process, thereby leading to a less effective frozen kernel. Meanwhile, since the frozen
value-block optimization converges quickly, it is still necessary to reserve
enough Phase~2 rounds after freezing, suggesting that $\tau$ should balance
kernel quality and post-freezing value-block convergence.

Indeed, the terms $B_\tau^{\lambda_0}$ and $R_\tau^{\lambda_0}$ depend on
the profile objective and the unknown optimum, and are therefore not available
during training. In practice, we treat $T_{\rm warm}$ as a hyperparameter and
select it by validation over a small set of candidate warm-up ratios.

\section{Experiments}
\subsection{Simulation Results}
In this section, we conduct simulation experiments to validate our theoretical
analysis. We consider a federated linear-attention task with $K=10$ clients, where all clients collaboratively train a single shared global model. To generate non-IID client data, we directly sample token embeddings
$H\in\mathbb{R}^{n\times d}$ instead of raw inputs. For client $k$, each
example is sampled from $H\sim\mathcal{N}(\mu_k,I_d)$, where $\mu_k\sim\mathcal{N}(0,\rho^2 I_d)$. The parameter $\rho$ controls the degree
of data heterogeneity: $\rho=0$ gives identical input distributions across
clients, while larger $\rho$ induces stronger covariate shift and hence more
severe client drift during local training. The results are shown below:

\begin{figure}[htbp!]
    \centering
    \includegraphics[width=\textwidth]{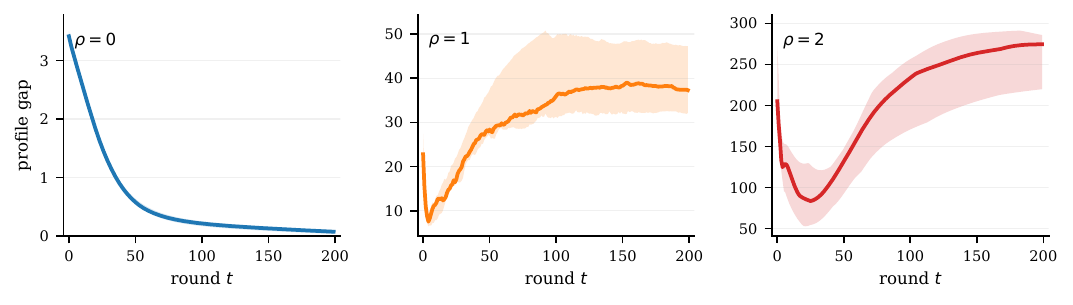}
    \caption{
    Simulation results for Phase-1 profile descent under different client
    heterogeneity levels. From left to right: profile objective
    $h_\lambda(\Phi^t)$ for $\rho=0$, $\rho=1$, and $\rho=2$.
    }
    \label{fig:sim_phase1_profile}
\end{figure}

\begin{figure}[htbp!]
    \centering
    \includegraphics[width=\textwidth]{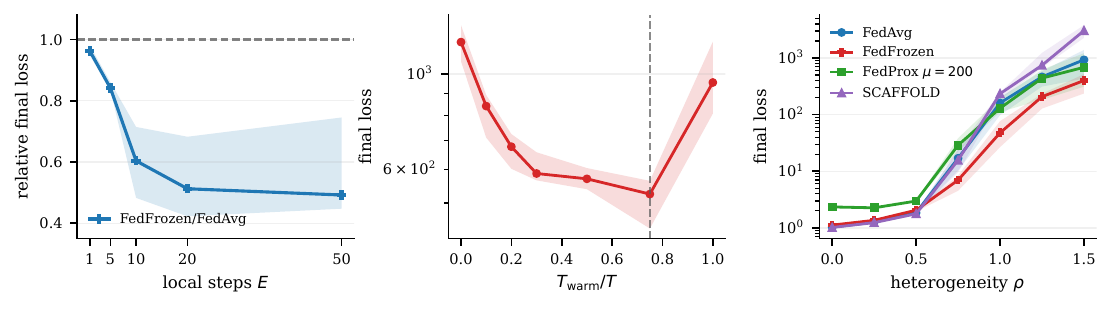}
    \caption{
    Simulation results for local-update sensitivity, warm-up selection, and
    robustness to client heterogeneity. From left to right: relative final
    loss of \modelname against FedAvg as the number of local steps $E$
    increases, bias--drift trade-off over warm-up length, and sensitivity to
    $\rho$. Values below one in the left panel indicate lower loss than FedAvg.
    }
    \label{fig:sim_main_results}
\end{figure}

Figure~\ref{fig:sim_phase1_profile} presents the trend of profile objective
$h_\lambda(\Phi^t)$ under different heterogeneity levels $\rho$ during the warm-up phase. In the homogeneous
case where $\rho=0$, $h_\lambda(\Phi^t)$ decreases steadily, showing that warm-up phase
continues to improve the shared attention kernel. In contrast, under stronger
heterogeneity, $h_\lambda(\Phi^t)$ decreases initially but later increases,
and this effect becomes more evident as $\rho$ grows. This trend aligns with the theoretical guarantee in Theorem~\ref{thm1} where we stated that Phase~1 behaves as an inexact descent process whose progress can eventually be offset by tracking error and heterogeneity-induced client drift.

Figure~\ref{fig:sim_main_results} validates the effectiveness of our two-stage design. The left panel reports the final-loss ratio of \modelname to FedAvg \citep{mcmahan2017communication} as the number
of local steps $E$ increases. Larger $E$ amplifies client drift in standard local
training, and the decreasing ratio shows that the advantage of \modelname
becomes more pronounced when local updates are stronger. The middle
panel evaluates different choices of $T_{\rm warm}/T$. Freezing too early
leads to a poorly trained attention kernel, while freezing too late keeps the
query/key block exposed to heterogeneous local updates for more rounds and
leaves fewer rounds for value-block optimization, which is consistent with Proposition~\ref{prop:e2e_certificate}. The right panel compares \modelname with FedAvg
\citep{mcmahan2017communication}, FedProx \citep{li2020fedprox}, and SCAFFOLD
\citep{karimireddy2020scaffold} as the heterogeneity level $\rho$ increases.
While the baselines are competitive in the nearly homogeneous region, their final losses deteriorate more rapidly under stronger heterogeneity.  In contrast, \modelname remains
substantially more stable, supporting the claim that freezing the query/key
kernel mitigates a structurally important source of client drift in federated
attention training.

\subsection{Real Data Results}
We further evaluate \modelname on three real-world image classification
benchmarks: CIFAR-10 and CIFAR-100 \citep{krizhevsky2009learning}, and
FEMNIST from LEAF \citep{caldas2018leaf}. To simulate statistical
heterogeneity across clients, we construct class-wise Dirichlet non-IID
partitions with concentration parameter $\alpha\in\{0.1,0.3,0.5\}$, where
smaller $\alpha$ corresponds to stronger label-distribution skew. For CIFAR-10
and CIFAR-100, we use ImageNet-pretrained ViT-B/32 backbones
\citep{dosovitskiy2021image} with dataset-specific classification heads. For
FEMNIST, we use an ImageNet-pretrained ViT-Small/16 backbone with a 62-way
classification head. All methods are evaluated using the same client
partitions, pretrained initialization, preprocessing pipeline, communication
budget, and local training budget.

We compare \modelname with representative federated optimization baselines,
including FedAvg, FedProx, SCAFFOLD, FedAvgM, FedAdam, and FedNova. Table~\ref{tab:real_data_results} reports the test accuracy under different
heterogeneity levels. Overall, \modelname delivers stable and competitive
performance across all real-data benchmarks. While the best baseline differs
across settings, \modelname consistently remains comparable to the strongest
methods and avoids severe degradation under challenging non-IID partitions.
These results indicate that query/key freezing after warm-up can preserve
predictive performance while improving the robustness of federated Transformer
training.

\begin{table}[htbp!]
\centering
\scriptsize
\setlength{\tabcolsep}{2.5pt}
\caption{Test accuracy on real-data benchmarks under class-wise Dirichlet non-IID partitions. Smaller $\alpha$ indicates stronger label heterogeneity.}
\label{tab:real_data_results}
\resizebox{\textwidth}{!}{%
\arrayrulecolor{black!25}
\begin{tabular}{@{}lccccccccc@{}}
\arrayrulecolor{black}
\toprule
Method & \multicolumn{3}{c}{CIFAR-10}
& \multicolumn{3}{c}{FEMNIST}
& \multicolumn{3}{c}{CIFAR-100} \\
\cmidrule(lr){2-4}\cmidrule(lr){5-7}\cmidrule(l){8-10}
& $\alpha=0.1$ & $\alpha=0.3$ & $\alpha=0.5$
& $\alpha=0.1$ & $\alpha=0.3$ & $\alpha=0.5$
& $\alpha=0.1$ & $\alpha=0.3$ & $\alpha=0.5$ \\
\midrule
\begin{tabular}[c]{@{}l@{}}FedAvg\\ \citep{mcmahan2017communication}\end{tabular}
& 0.9070 & 0.9270 & 0.9385 & 0.7880 & 0.8070 & 0.8100 & 0.6450 & 0.6725 & 0.6860 \\
\arrayrulecolor{black!25}\hdashline\arrayrulecolor{black}
\begin{tabular}[c]{@{}l@{}}FedProx\\ \citep{li2020fedprox}\end{tabular}
& 0.9150 & 0.9395 & 0.9510 & 0.7900 & 0.8100 & 0.8080 & 0.6160 & 0.6545 & 0.6580 \\
\arrayrulecolor{black!25}\hdashline\arrayrulecolor{black}
\begin{tabular}[c]{@{}l@{}}SCAFFOLD\\ \citep{karimireddy2020scaffold}\end{tabular}
& 0.9295 & 0.9475 & 0.9560 & 0.7300 & 0.7580 & 0.7700 & 0.3880 & 0.4480 & 0.4730 \\
\arrayrulecolor{black!25}\hdashline\arrayrulecolor{black}
\begin{tabular}[c]{@{}l@{}}FedAvgM\\ \citep{hsu2019measuringeffectsnonidenticaldata}\end{tabular}
& 0.6015 & 0.6540 & 0.7030 & 0.7540 & 0.7700 & 0.7930 & 0.4875 & 0.4865 & 0.4905 \\
\arrayrulecolor{black!25}\hdashline\arrayrulecolor{black}
\begin{tabular}[c]{@{}l@{}}FedAdam\\ \citep{reddi2021adaptivefederatedoptimization}\end{tabular}
& 0.7445 & 0.8930 & 0.9105 & 0.7580 & 0.7710 & 0.7860 & 0.4135 & 0.5015 & 0.4870 \\
\arrayrulecolor{black!25}\hdashline\arrayrulecolor{black}
\begin{tabular}[c]{@{}l@{}}FedNova\\ \citep{wang2020fednova}\end{tabular}
& 0.9155 & 0.9260 & 0.9440 & 0.8060 & 0.8070 & 0.8160 & 0.6380 & 0.6735 & 0.6835 \\
\arrayrulecolor{black!25}\hdashline\arrayrulecolor{black}
\addlinespace[1.5pt]
\modelname
& 0.9040 & 0.9325 & 0.9470 & 0.7890 & 0.8100 & 0.8120 & 0.6690 & 0.6790 & 0.6955 \\
\arrayrulecolor{black}
\bottomrule
\end{tabular}%
}
\end{table}

\begin{table}[htbp!]
\centering
\scriptsize
\setlength{\tabcolsep}{2.5pt}
\caption{Lowest measured communication-cost ratios of \modelname relative to FedAvg across pretrained Transformer architectures. Lower ratios indicate greater communication savings.}
\label{tab:fedfrozen_comm_cost_top5}
\resizebox{\textwidth}{!}{%
\arrayrulecolor{black!25}
\begin{tabular}{@{}lccccc@{}}
\arrayrulecolor{black}
\toprule
Metric
& \begin{tabular}[c]{@{}c@{}}ViT-B/32\\ \citep{dosovitskiy2021image}\end{tabular}
& \begin{tabular}[c]{@{}c@{}}T5-Small\\ \citep{raffel2020exploring}\end{tabular}
& \begin{tabular}[c]{@{}c@{}}BART-Base\\ \citep{lewis2020bart}\end{tabular}
& \begin{tabular}[c]{@{}c@{}}ELECTRA-Base\\ \citep{clark2020electra}\end{tabular}
& \begin{tabular}[c]{@{}c@{}}BERT-Base\\ \citep{devlin2019bert}\end{tabular} \\
\midrule
QK Params (\%)
& 16.10 & 10.10 & 9.81 & 13.02 & 12.95 \\
\arrayrulecolor{black!25}\hdashline\arrayrulecolor{black}
QKV Params (\%)
& 24.15 & 15.16 & 14.72 & 19.53 & 19.42 \\
\arrayrulecolor{black!25}\hdashline\arrayrulecolor{black}
\addlinespace[1.5pt]
Cost Ratio
& 0.8712 & 0.8752 & 0.8780 & 0.8959 & 0.8964 \\
\arrayrulecolor{black}
\bottomrule
\end{tabular}%
}
\end{table}

In addition to predictive performance, \modelname also improves communication
efficiency. Regularization-based methods such as FedProx usually have
communication costs comparable to FedAvg, while SCAFFOLD incurs roughly twice
the communication cost of FedAvg due to the introduction of additional control
variates. In contrast, the communication reduction of \modelname comes from the
model architecture itself: architectures with a larger
query/key parameter ratio can achieve greater communication reduction after
freezing these components. Table~\ref{tab:fedfrozen_comm_cost_top5} reports the communication-cost
ratios of \modelname relative to FedAvg on several commonly used pretrained
Transformer architectures. We observe that \modelname reduces the communication
cost by at least $10\%$ across these architectures, showing that the proposed
freezing strategy can provide practical communication savings without adding
extra correction variables or changing the local training objective.

\section{Conclusion}
In this paper, we introduced \modelname, a two-stage federated optimization framework designed to mitigate client drift in attention-based models. Theoretically, we established that the full-model warm-up acts as an inexact descent on a kernel-profile objective, while the subsequent frozen stage efficiently solves a restricted value-block optimization problem. This analysis explicitly characterizes the fundamental trade-off among kernel quality, freezing bias, and the accumulation of client drift. Empirically, evaluations on both simulated environments and real-world benchmarks demonstrate that \modelname enhances training robustness under data heterogeneity, while concurrently offering practical reductions in communication costs. While our current theoretical guarantees are grounded in a linear-attention formulation, the empirical success of \modelname on standard Transformers highlights its broader applicability. Future work will focus on extending our analytical framework to standard nonlinear attention mechanisms and exploring the potential of component-wise freezing strategies for the distributed fine-tuning of large language models (LLMs).

\bibliographystyle{plainnat}
\bibliography{reference}

\newpage
\appendix
\setcounter{thm}{0}
\setcounter{col}{0}
\setcounter{lem}{0}
\setcounter{defn}{0}
\setcounter{prop}{0}
\setcounter{remark}{0}

\begin{center}
    \interlinepenalty=10000
    {\LARGE \bf Appendix}
\end{center}
In the appendix, we provide the mathematical proofs for all theorems presented in the paper, followed by additional experimental results.

\section{Proofs}
\subsection{Proof of Theorem~\ref{thm1}}
\begin{lem}\label{lm0}
    Define $Z(H; \Phi) = \varphi(HW^Q)\varphi(HW^K)^\top H$. Assume Assumption~\ref{asmp:attention_assumption} holds. There exists $C_1,C_2,C_3,C_4>0$ such that:
        \begin{enumerate}[label=(1.\arabic*),leftmargin=3em]
            \item
            $
            \left\| Z(H; \Phi)\right\|_F \le C_1 ,
            $
            \item
            $
            \left\|Z(H; \Phi) - Z(H; \Phi')\right\|_F \le C_2\left\|\Phi-\Phi'\right\|_F ,
            $
            \item
            $
            \left\| D_\Phi O(H;\Phi,\theta)\right\|_{\mathrm{op}} \le C_3 ,
            $
            \item
            $
            \left\| D_\Phi O(H;\Phi,\theta) - D_\Phi O(H;\Phi',\theta)\right\|_{\mathrm{op}} \le C_4\left\|\Phi-\Phi'\right\|_F
            $
        \end{enumerate}
    hold for any kernels $\Phi,\Phi'$ and any value block $\theta$ satisfying $\norm{\theta}_F \le B_\theta$.
\end{lem}

\begin{proof}
    For (1.1),
    \begin{equation}
        \begin{aligned}
            \left\|Z(H;\Phi)\right\|_F & = \left\|\varphi(HW^Q)\varphi(HW^K)^\top H\right\|_F \\
            &\le \left\| \varphi(HW^Q)\right\|_F \cdot \left\|\varphi(HW^K)\right\|_F \cdot \left\|H\right\|_F \\
            &\le B_\varphi^2R \\
            &\triangleq C_1.
        \end{aligned}
    \end{equation}
    For (1.2), let $\Phi' = (W^{Q'}, W^{K'})$,
    \begin{equation}
        \begin{aligned}
            \left\|Z(H;\Phi)-Z(H;\Phi')\right\|_F  &= \left\|\varphi(HW^Q)\varphi(HW^K)^\top H - \varphi(HW^{Q'})\varphi(HW^{K'})^\top H\right\|_F \\
            &\le \left\|\left(\varphi(HW^Q) - \varphi(HW^{Q'})\right)\varphi(HW^K)^\top H\right\|_F \\
            & \quad + \left\|\varphi(HW^{Q'})\left(\varphi(HW^K) - \varphi(HW^{K'})\right)^\top H\right\|_F \\
            &\le \left\|\varphi(HW^Q) - \varphi(HW^{Q'})\right\|_F\left\|\varphi(HW^K)\right\|_F\left\|H\right\|_F \\
            &\quad + \left\|\varphi(HW^{Q'})\right\|_F\left\|\varphi(HW^K) - \varphi(HW^{K'})\right\|_F\left\|H\right\|_F \\
            &\le \left\|D\varphi\right\|_{\mathrm{op}}\left\|HW^Q-HW^{Q'}\right\|_F\left\|\varphi(HW^K)\right\|_F\left\|H\right\|_F \\
            &\quad + \left\|\varphi(HW^{Q'})\right\|_F\left\|D\varphi\right\|_{\mathrm{op}}\left\|HW^K-HW^{K'}\right\|_F\left\|H\right\|_F \\
            &\le B_\varphi J_\varphi R^2\left(\left\|W^Q-W^{Q'}\right\|_F + \left\|W^K-W^{K'}\right\|_F\right) \\
            &\le \sqrt{2}B_\varphi J_\varphi R^2\left\| \Phi-\Phi'\right\|_F \\
            &\triangleq C_2\left\|\Phi-\Phi'\right\|_F.
        \end{aligned}
    \end{equation}
    For (1.3),
    \begin{equation}
        \begin{aligned}
            \left\| D_\Phi O(H;\Phi,\theta)\right\|_{\mathrm{op}} & = \mathop{\sup}_{\|U\|_F=1}\left\|D_\Phi Z(H;\Phi)[U]\cdot\theta\right\|_F \\
            &\le \mathop{\sup}_{\|U\|_F=1}\left\|D_\Phi Z(H;\Phi)[U]\right\|_F \cdot \left\|\theta\right\|_F \\
            &\le C_2B_\theta \\
            &\triangleq C_3.
        \end{aligned}
    \end{equation}
    For (1.4),
    \begin{equation}
        \begin{aligned}
            &\left\| D_\Phi O(H;\Phi,\theta) - D_\Phi O(H;\Phi',\theta)\right\|_{\mathrm{op}}\\
            =& \left\| D_\Phi \left(\varphi(HW^Q)\varphi(HW^K)^\top H\theta\right) - D_\Phi\left(\varphi(HW^{Q'})\varphi(HW^{K'})^\top H\theta\right)\right\|_{\mathrm{op}} \\
            \le &\left\| D_\Phi \left(\varphi(HW^Q)\varphi(HW^K)^\top H\theta\right) - D_\Phi\left(\varphi(HW^{Q'})\varphi(HW^{K})^\top H\theta\right)\right\|_{\mathrm{op}} \\
            & + \left\| D_\Phi \left(\varphi(HW^{Q'})\varphi(HW^K)^\top H\theta\right) - D_\Phi\left(\varphi(HW^{Q'})\varphi(HW^{K'})^\top H\theta\right)\right\|_{\mathrm{op}} \\
            =& \mathop{\sup}_{\|U\|_F =1}\left\| D_\Phi \left(\left(\varphi(HW^Q) - \varphi(HW^{Q'})\right)\varphi(HW^K)^\top H\theta\right)[U]\right\|_F \\
            & +\mathop{\sup}_{\|U\|_F =1} \left\| D_\Phi \left(\varphi(HW^{Q'})\left(\varphi(HW^K)-\varphi(HW^{K'})\right)^\top H\theta\right)[U]\right\|_F \\
            \le&  \mathop{\sup}_{\|U\|_F =1}\left\|D_\Phi\left(\varphi(HW^Q) - \varphi(HW^{Q'})\right)[U]\cdot \varphi(HW^K)^\top H\theta \right\|_F \\
            & + \mathop{\sup}_{\|U\|_F =1}\left\|\left(\varphi(HW^Q) - \varphi(HW^{Q'})\right)\cdot D_\Phi\varphi(HW^K)^\top [U]\cdot H\theta \right\|_F \\
             &+ \mathop{\sup}_{\|U\|_F =1} \left\| D_\Phi\varphi(HW^{Q'})[U]\cdot \left(\varphi(HW^K)-\varphi(HW^{K'})\right)^\top H\theta\right\|_F \\
             &+ \mathop{\sup}_{\|U\|_F =1} \left\| \varphi(HW^{Q'})\cdot D_\Phi \left(\varphi(HW^K)-\varphi(HW^{K'})\right)^\top[U]\cdot H\theta\right\|_F \\
             \le & M_\varphi \left\|HW^Q-HW^{Q'}\right\|_F \left\|H\right\|_F\left\|\varphi(HW^K)\right\|_F\left\|H\right\|_F\left\|\theta\right\|_F \\
             &+ J_\varphi \left\|HW^Q-HW^{Q'}\right\|_F\cdot J_\varphi \left\|H\right\|_F^2\left\|\theta\right\|_F \\
             &+ J_\varphi \left\|H\right\|_F \cdot J_\varphi\left\|HW^K-HW^{K'}\right\|_F\left\|H\right\|_F\left\|\theta\right\|_F \\
             &+ \left\| \varphi(HW^{Q'})\right\|_F M_\varphi\left\|HW^K-HW^{K'}\right\|_F\left\|H\right\|_F\left\|\theta\right\|_F \\
             \le& \left(M_\varphi B_\varphi B_\theta R^3 + B_\theta J_\varphi^2R^3\right) \left(\left\|W^Q-W^{Q'}\right\|_F + \left\|W^K-W^{K'}\right\|_F\right) \\
             \le& \sqrt{2}\left(M_\varphi B_\varphi B_\theta R^3 + B_\theta J_\varphi^2R^3\right)\left\|\Phi-\Phi'\right\|_F \\
             \triangleq& C_4\left\|\Phi-\Phi'\right\|_F.
        \end{aligned}
    \end{equation}
\end{proof}
\begin{lem}\label{lm1}
    Assume Assumption~\ref{asmp:attention_assumption} holds. There exist constants $L_{\theta\theta}, L_{\theta\Phi}, L_{\Phi\theta}, L_{\Phi\Phi}>0$, such that for every client $k$ and any $(\Phi, \theta), (\Phi', \theta')$ satisfying $\norm{\theta}_F \le B_\theta$ and $\norm{\theta'}_F\le B_\theta$,
    \begin{equation}\label{lm1-1}
        \begin{aligned}
            &\left\| \nabla_\theta F_k(\Phi,\theta)-\nabla_\theta F_k(\Phi,\theta') \right\|_F \le L_{\theta\theta}\left\|\theta-\theta'\right\|_F , \\
            &\left\| \nabla_\theta F_k(\Phi,\theta)-\nabla_\theta F_k(\Phi',\theta) \right\|_F \le L_{\theta\Phi}\left\|\Phi-\Phi'\right\|_F , \\
            &\left\| \nabla_\Phi F_k(\Phi,\theta)-\nabla_\Phi F_k(\Phi,\theta') \right\|_F \le L_{\Phi\theta}\left\|\theta-\theta'\right\|_F , \\
            &\left\| \nabla_\Phi F_k(\Phi,\theta)-\nabla_\Phi F_k(\Phi',\theta) \right\|_F \le L_{\Phi\Phi}\left\|\Phi-\Phi'\right\|_F .
        \end{aligned}
    \end{equation}

\end{lem}
\begin{proof}
    By Lemma \ref{lm0}, for the $\theta$-$\theta$ Lipschitz bound,
    \begin{equation}
        \begin{aligned}
            &\left\| \nabla_\theta F_k(\Phi,\theta)-\nabla_\theta F_k(\Phi,\theta') \right\|_F \\
            \le &\mathbb{E}_{\mathcal{D}_k} \left[\left\| Z(H;\Phi)\right\|_F\cdot \left\|\nabla_O \ell \left(O\left(H;\Phi,\theta\right),y\right) - \nabla_O \ell \left(O\left(H;\Phi,\theta'\right),y\right)\right\|_F\right]\\
            \le & C_1\cdot L_OC_1\left\|\theta-\theta'\right\|_F \\
            =&C_1^2L_O\left\|\theta-\theta'\right\|_F.
        \end{aligned}
    \end{equation}
    For the $\theta$-$\Phi$ Lipschitz bound,
    \begin{equation}
        \begin{aligned}
            &\left\| \nabla_\theta F_k(\Phi,\theta)-\nabla_\theta F_k(\Phi',\theta) \right\|_F \\
            \le &\mathbb{E}_{\mathcal{D}_k} \left[\left\| Z(H;\Phi)\right\|_F\cdot \left\|\nabla_O \ell \left(O\left(H;\Phi,\theta\right),y\right) - \nabla_O \ell \left(O\left(H;\Phi',\theta\right),y\right)\right\|_F\right]\\
            &+\mathbb{E}_{\mathcal{D}_k} \left[\left\| Z(H;\Phi)- Z(H;\Phi')\right\|_F \cdot\left\|\nabla_O \ell \left(O\left(H;\Phi',\theta\right),y\right)\right\|_F\right] \\
            \le& C_1\cdot L_OC_2\left\|\Phi-\Phi'\right\|_FB_\theta + C_2\left\|\Phi-\Phi'\right\|_F\cdot G_O\\
            =& \left(C_1C_2L_O B_\theta +C_2G_O\right)\left\|\Phi-\Phi'\right\|_F.
        \end{aligned}
    \end{equation}
    For the $\Phi$-$\theta$ Lipschitz bound,
    \begin{equation}
        \begin{aligned}
            &\left\| \nabla_\Phi F_k(\Phi,\theta)-\nabla_\Phi F_k(\Phi,\theta') \right\|_F \\
            \le& \mathbb{E}_{\mathcal{D}_k} \left[\left\|D_\Phi O\left(H;\Phi,\theta\right) - D_\Phi O\left(H; \Phi,\theta'\right)\right\|_{\mathrm{op}}\cdot \left\|\nabla_O \ell \left(O\left(H;\Phi,\theta\right),y\right)\right\|_F\right] \\
            &+\mathbb{E}_{\mathcal{D}_k} \left[\left\|D_\Phi O\left(H; \Phi,\theta'\right)\right\|_{\mathrm{op}}\cdot\left\|\nabla_O \ell \left(O\left(H;\Phi,\theta\right),y\right) - \nabla_O \ell \left(O\left(H;\Phi,\theta'\right),y\right)\right\|_F\right] \\
            \le&C_2 \left\|\theta-\theta'\right\|_F\cdot G_O + C_2B_\theta\cdot L_OC_1 \left\|\theta-\theta'\right\|_F \\
            =& C_2\left(G_O+C_1L_OB_\theta\right)\left\|\theta-\theta'\right\|_F.
        \end{aligned}
    \end{equation}
    For the $\Phi$-$\Phi$ Lipschitz bound,
    \begin{equation}
        \begin{aligned}
            &\left\| \nabla_\Phi F_k(\Phi,\theta)-\nabla_\Phi F_k(\Phi',\theta) \right\|_F \\
            \le&\mathbb{E}_{\mathcal{D}_k} \left[\left\|D_\Phi O\left(H;\Phi,\theta\right) - D_\Phi O\left(H; \Phi',\theta\right)\right\|_{\mathrm{op}}\cdot \left\|\nabla_O \ell \left(O\left(H;\Phi,\theta\right),y\right)\right\|_F\right] \\
            &+\mathbb{E}_{\mathcal{D}_k} \left[\left\|D_\Phi O\left(H; \Phi',\theta\right)\right\|_{\mathrm{op}}\cdot\left\|\nabla_O \ell \left(O\left(H;\Phi,\theta\right),y\right) - \nabla_O \ell \left(O\left(H;\Phi',\theta\right),y\right)\right\|_F\right] \\
            \le& C_4\left\|\Phi-\Phi'\right\|_F\cdot G_O + C_3\cdot L_OC_2\left\|\Phi-\Phi'\right\|_FB_\theta \\
            =&\left(C_4G_O+C_2C_3L_OB_\theta\right)\left\|\Phi-\Phi'\right\|_F.
        \end{aligned}
    \end{equation}
    Let,
    \begin{equation}
        \begin{aligned}
            &L_{\theta\theta} = C_1^2L_O\\
            &L_{\theta\Phi}= C_2\left(G_O+C_1L_OB_\theta\right)\\
            &L_{\Phi\theta} = C_2\left(G_O+C_1L_OB_\theta\right)\\
            &L_{\Phi\Phi} = C_4G_O+C_2C_3L_OB_\theta
        \end{aligned}
    \end{equation}
    then Lemma \ref{lm1} holds.
\end{proof}

\begin{lem}\label{lm:theta_star_bound}
    Assume Assumption~\ref{asmp:attention_assumption} holds, when $\lambda \ge C_1G_O/B_\theta$, for any kernel $\Phi$,
    \begin{equation}
        \norm{\theta_\lambda^\star(\Phi)}_F \le B_\theta
    \end{equation}
\end{lem}
\begin{proof}
    Since $\theta_\lambda^\star(\Phi)$ is the minimizer, according to the first order condition, we have:
    \begin{equation}
        0 = \nabla_\theta f_\lambda(\Phi,\theta_\lambda^\star(\Phi)) = \nabla_\theta f(\Phi,\theta_\lambda^\star(\Phi)) + \lambda \theta_\lambda^\star(\Phi)
    \end{equation}
    then,
    \begin{equation}
        \norm{\theta_\lambda^\star(\Phi)}_F = \frac1 \lambda  \norm{\nabla_\theta f(\Phi,\theta_\lambda^\star(\Phi))}_F \le \frac{C_1G_O}{\lambda} \le B_\theta.
    \end{equation}
\end{proof}

\begin{lem}\label{lm2}
    Assume  Assumption~\ref{asmp:attention_assumption} holds and $\lambda \ge C_1G_O/B_\theta$. Then $\theta_\lambda^\star$ is unique and Lipschitz with respect to $\Phi$, i.e., for any kernels $\Phi,\Phi'$:
    \begin{equation}\label{lm2-1}
        \left\| \theta_\lambda^\star(\Phi) - \theta_\lambda^\star(\Phi')\right\|_F \le \frac{L_{\theta\Phi}}{\lambda} \left\|\Phi-\Phi'\right\|_F
    \end{equation}
    furthermore, $h_\lambda$ is $L_h^\lambda$-smooth, i.e., for any kernels $\Phi, \Phi'$:
    \begin{equation}\label{lm2-2}
        \left\| \nabla h_\lambda(\Phi) - \nabla h_\lambda(\Phi')\right\|_F \le L_h^\lambda \left\| \Phi - \Phi' \right\|_F
    \end{equation}
    where $L_h^\lambda = L_{\Phi\Phi} + \lambda^{-1}L_{\Phi\theta}L_{\theta\Phi}$.
\end{lem}

\begin{proof}
    Note that for any kernel $\Phi$, $f_\lambda(\Phi, \theta)$ is $\lambda$-strongly convex with respect to $\theta$, so that $\theta_\lambda^\star$ is unique.

    For any kernels $\Phi,\Phi'$, by the definition of $\theta_\lambda^\star$:
    \begin{equation}
        \nabla_\theta f_\lambda\left(\Phi, \theta_\lambda^\star(\Phi)\right) = 0 = \nabla_\theta f_\lambda\left(\Phi', \theta_\lambda^\star(\Phi')\right)
    \end{equation}
    then by the $\lambda$-strongly convex property and Lemmas~\ref{lm1} and~\ref{lm:theta_star_bound}, we have:
    \begin{equation}
        \begin{aligned}
            \left\|\theta_\lambda^\star(\Phi)-\theta_\lambda^\star(\Phi')\right\|_F &\le \frac1\lambda \left\| \nabla_\theta f_\lambda\left(\Phi,\theta_\lambda^\star(\Phi)\right) - \nabla_\theta f_\lambda\left(\Phi,\theta_\lambda^\star(\Phi')\right)\right\|_F \\
            &= \frac1\lambda \left\| \nabla_\theta f_\lambda\left(\Phi',\theta_\lambda^\star(\Phi')\right) - \nabla_\theta f_\lambda\left(\Phi,\theta_\lambda^\star(\Phi')\right)\right\|_F \\
            &\le \frac{L_{\theta\Phi}}{\lambda} \left\|\Phi-\Phi'\right\|_F.
        \end{aligned}
    \end{equation}
    For the smoothness of $h_\lambda$, by Danskin's theorem,
    \begin{equation}
        \nabla h_\lambda(\Phi) = \nabla_\Phi f_\lambda\left(\Phi, \theta_\lambda^\star(\Phi)\right)
    \end{equation}
    then for any kernels $\Phi,\Phi'$:
    \begin{equation}
        \begin{aligned}
            \left\|\nabla h_\lambda(\Phi) - \nabla h_\lambda(\Phi')\right\|_F &= \left\|\nabla_\Phi f_\lambda\left(\Phi, \theta_\lambda^\star(\Phi)\right) - \nabla_\Phi f_\lambda\left(\Phi', \theta_\lambda^\star(\Phi')\right)\right\|_F \\
            &\le \left\|\nabla_\Phi f_\lambda\left(\Phi, \theta_\lambda^\star(\Phi)\right) - \nabla_\Phi f_\lambda\left(\Phi, \theta_\lambda^\star(\Phi')\right)\right\|_F \\
            &\quad + \left\|\nabla_\Phi f_\lambda\left(\Phi, \theta_\lambda^\star(\Phi')\right) - \nabla_\Phi f_\lambda\left(\Phi', \theta_\lambda^\star(\Phi')\right)\right\|_F \\
            &\le L_{\Phi\theta} \left\|\theta_\lambda^\star(\Phi)-\theta_\lambda^\star(\Phi')\right\|_F + L_{\Phi\Phi} \left\| \Phi - \Phi' \right\|_F \\
            &\le \left(\frac{L_{\Phi\theta}L_{\theta\Phi}}{\lambda}+ L_{\Phi\Phi}\right)\left\| \Phi - \Phi' \right\|_F \\
            &= L_h^\lambda \left\| \Phi - \Phi' \right\|_F
        \end{aligned}
    \end{equation}
    therefore (\ref{lm2-2}) holds.
\end{proof}

\begin{lem}\label{lm3}
    Assume Assumption \ref{asmp:attention_assumption} and \ref{asmp:global_dissim} hold. Let,
    \begin{equation}\label{lm3-1}
            b_\Phi^t = \frac1E \sum\limits_{k=1}^K p_k\sum\limits_{j=0}^{E-1}\left(\nabla_\Phi F_k^\lambda\left(\Phi_{k,j}^t,\theta_{k,j}^t\right) - \nabla_{\Phi}F_k^\lambda\left(\Phi^t,\theta^t\right)\right)
    \end{equation}
    let $L_f^\lambda  =\sqrt{2}\cdot \max\left\{\sqrt{L_{\Phi\Phi}^2 + L_{\theta\Phi}^2}, \sqrt{L_{\Phi\theta}^2+(L_{\theta\theta}+\lambda)^2}\right\}$, for any $\epsilon>0$, when $\eta L_f^\lambda E \le \epsilon /(1+\epsilon)$, we have:
    \begin{equation}
        \left\|b_\Phi^t\right\|_F \le \epsilon\sqrt{M_\lambda^2 + B^2 \left\|\nabla f_\lambda (\Phi^t,\theta^t)\right\|_F^2}
    \end{equation}
    where $M_\lambda^2 = M^2+2(B^2-1)\lambda\cdot C_1G_OB_\theta.$
\end{lem}
\begin{proof}
    Let,
    \begin{equation}
        c_{k,j}^t = (\Phi_{k,j}^t,\theta_{k,j}^t) - (\Phi^t,\theta^t)
    \end{equation}
    where $0\le t \le T_\mathrm{warm}-1, 1\le k\le K, 0\le j \le E$. Then $c_{k,0}^t = 0$ and for $0 \le j \le E-1$,
    \begin{equation}
        \begin{aligned}
            c_{k,j+1}^t &= c_{k,j}^t - \eta \nabla F_{k}^\lambda (\Phi_{k,j}^t,\theta_{k,j}^t) \\
            & = c_{k,j}^t - \eta \left(\nabla F_{k}^\lambda (\Phi_{k,j}^t,\theta_{k,j}^t)-\nabla F_{k}^\lambda (\Phi^t,\theta^t)\right) -  \eta \nabla F_{k}^\lambda (\Phi^t,\theta^t)
        \end{aligned}
    \end{equation}
    By Lemma \ref{lm1}, for any $(\Phi,\theta), (\Phi',\theta')$ with $\norm{\theta}_F, \norm{\theta'}_F \le B_\theta $ and $1\le k\le K$,
    \begin{equation}
        \begin{aligned}
            \left\|\nabla F_{k}^\lambda (\Phi,\theta)-\nabla F_{k}^\lambda (\Phi',\theta')\right\|_F &\le \left\|\nabla F_{k}^\lambda (\Phi,\theta)-\nabla F_{k}^\lambda (\Phi',\theta)\right\|_F + \left\|\nabla F_{k}^\lambda (\Phi',\theta)-\nabla F_{k}^\lambda (\Phi',\theta')\right\|_F \\
            &\le \sqrt{L_{\Phi\Phi}^2 + L_{\theta\Phi}^2}\left\|\Phi-\Phi'\right\|_F + \sqrt{L_{\Phi\theta}^2+(L_{\theta\theta}+\lambda)^2}\left\|\theta-\theta'\right\|_F \\
            &\le L_f^\lambda  \left\|(\Phi,\theta)-(\Phi',\theta')\right\|_F
        \end{aligned}
    \end{equation}
    then $F_{k}^\lambda$ is $L_f^\lambda$-smooth, so that for $0 \le j \le E-1$,
    \begin{equation}
        \begin{aligned}
            \left\|c_{k,j+1}^t\right\|_F &=\left\|c_{k,j}^t - \eta \left(\nabla F_{k}^\lambda (\Phi_{k,j}^t,\theta_{k,j}^t)-\nabla F_{k}^\lambda (\Phi^t,\theta^t)\right) -  \eta \nabla F_{k}^\lambda (\Phi^t,\theta^t)\right\|_F \\
            &\le (1+\eta L_f^\lambda)\left\|c_{k,j}^t\right\|_F + \eta \left\|\nabla F_{k}^\lambda (\Phi^t,\theta^t)\right\|_F
        \end{aligned}
    \end{equation}
    then for $1\le j \le E$:
    \begin{equation}
        \left\|c_{k,j}^t\right\|_F \le \frac{(1+\eta L_f^\lambda)^j-1}{L_f^\lambda} \left\|\nabla F_{k}^\lambda (\Phi^t,\theta^t)\right\|_F
    \end{equation}
    When $\eta L_f^\lambda E \le \epsilon /(1+\epsilon)$, we have $(1+ \eta L_f^\lambda)^E -1\le (1+\epsilon)\eta L_f^\lambda E$, then,
    \begin{equation}
        \begin{aligned}
             \left\|b_\Phi^t\right\|_F &\le \frac1E \sum\limits_{k=1}^K p_k\sum\limits_{j=0}^{E-1}\left\|\nabla_\Phi F_k^\lambda\left(\Phi_{k,j}^t,\theta_{k,j}^t\right) - \nabla_{\Phi}F_k^\lambda\left(\Phi^t,\theta^t\right)\right\|_F \\
             &\le \frac{L_f^\lambda}{E}\sum\limits_{k=1}^K p_k\sum\limits_{j=0}^{E-1}\left\|c_{k,j}^t\right\|_F \\
             &\le \frac{L_f^\lambda}{E}\sum\limits_{k=1}^K p_k\sum\limits_{j=0}^{E-1} \frac{(1+\eta L_f^\lambda)^j-1}{L_f^\lambda} \left\|\nabla F_{k}^\lambda (\Phi^t,\theta^t)\right\|_F \\
             &= \frac1E \sum\limits_{k=1}^K p_k \left( \frac{(1+ \eta L_f^\lambda)^E - 1}{\eta L_f^\lambda} - E\right)\left\|\nabla F_{k}^\lambda (\Phi^t,\theta^t)\right\|_F \\
             &\le \epsilon\sum\limits_{k=1}^K p_k\left\|\nabla F_{k}^\lambda (\Phi^t,\theta^t)\right\|_F
        \end{aligned}
    \end{equation}
    by Assumption \ref{asmp:attention_assumption} and \ref{asmp:global_dissim}, we have:
    \begin{equation}
        \begin{aligned}
            \left(\sum\limits_{k=1}^K p_k\left\|\nabla F_{k}^\lambda (\Phi^t,\theta^t)\right\|_F\right)^2 &\le \sum\limits_{k=1}^K p_k\left\|\nabla F_{k}^\lambda (\Phi^t,\theta^t)\right\|_F^2 \\
            &= \sum\limits_{k=1}^K p_k\left(\left\|\nabla F_{k} (\Phi^t,\theta^t)\right\|_F^2 + 2\lambda \langle \nabla_\theta F_{k} (\Phi^t,\theta^t),  \theta^t \rangle + \lambda^2 \left\|\theta^t\right\|_F^2\right) \\
            &= \sum\limits_{k=1}^K p_k\left\|\nabla F_{k} (\Phi^t,\theta^t)\right\|_F^2 + 2\lambda \langle \nabla_\theta f (\Phi^t,\theta^t),  \theta^t \rangle + \lambda^2 \left\|\theta^t\right\|_F^2 \\
            &\le M^2 + B^2 \left\|\nabla f (\Phi^t,\theta^t)\right\|_F^2 + 2\lambda \langle \nabla_\theta f (\Phi^t,\theta^t),  \theta^t \rangle + \lambda^2 \left\|\theta^t\right\|_F^2 \\
            &= M^2 + B^2 \left\|\nabla f_\lambda (\Phi^t,\theta^t)\right\|_F^2 + (1-B^2)\cdot\left(2\lambda \langle \nabla_\theta f (\Phi^t,\theta^t),  \theta^t \rangle + \lambda^2 \left\|\theta^t\right\|_F^2\right) \\
            &\le M^2 + B^2 \left\|\nabla f_\lambda (\Phi^t,\theta^t)\right\|_F^2+2(B^2-1)\lambda \left\|\nabla_\theta f(\Phi^t,\theta^t)\right\|_F\left\|\theta^t\right\|_F \\
            &\le M^2+2(B^2-1)\lambda\cdot C_1G_OB_\theta +  B^2 \left\|\nabla f_\lambda (\Phi^t,\theta^t)\right\|_F^2\\
            &= M_\lambda ^2 + B^2 \left\|\nabla f_\lambda (\Phi^t,\theta^t)\right\|_F^2
        \end{aligned}
    \end{equation}
    therefore,
    \begin{equation}
        \begin{aligned}
            \left\|b_\Phi^t\right\|_F \le \epsilon\sum\limits_{k=1}^K p_k\left\|\nabla F_{k}^\lambda (\Phi^t,\theta^t)\right\|_F \le \epsilon\sqrt{M_\lambda^2 + B^2 \left\|\nabla f_\lambda (\Phi^t,\theta^t)\right\|_F^2}
        \end{aligned}
    \end{equation}
\end{proof}
Now we are ready to prove Theorem~\ref{thm1}.
\begin{proof}
    For $0 \le t \le T_\mathrm{warm} - 1$,
    \begin{equation}
        \begin{aligned}
            \Phi^{t+1} & = \Phi^t - \eta \sum\limits_{k=1}^Kp_k\sum\limits_{j=0}^{E-1}\nabla_\Phi F_k^\lambda(\Phi_{k,j}^t,\theta_{k,j}^t) \\
            &=\Phi^t - \eta E \left(\nabla_\Phi f_{\lambda}(\Phi^t,\theta^t) + b_\Phi^t\right) \\
        \end{aligned}
    \end{equation}
    Let $r_\lambda^t = \nabla_\Phi f_{\lambda}(\Phi^t,\theta^t) - \nabla h_\lambda(\Phi^t)$, note that $h_\lambda$ is $L_h^\lambda$-smooth, when $\eta \le 1/(L_h^\lambda E)$, we have:
    \begin{equation}
        \begin{aligned}
            h_\lambda(\Phi^{t+1}) &\le h_\lambda(\Phi^{t}) + \langle \nabla h_\lambda(\Phi^{t}), \Phi^{t+1}-\Phi^t\rangle + \frac {L_h^\lambda} 2 \left\|\Phi^{t+1}-\Phi^t\right\|_F^2 \\
            &= h_\lambda(\Phi^{t}) - \eta E \langle \nabla h_\lambda(\Phi^{t}), \nabla_\Phi f_{\lambda}(\Phi^t,\theta^t) + b_\Phi^t\rangle + \frac{\eta^2 E^2L_h^\lambda}{2}\left\|\nabla_\Phi f_{\lambda}(\Phi^t,\theta^t) + b_\Phi^t\right\|_F^2 \\
            &= h_\lambda(\Phi^{t}) - \eta E \langle \nabla h_\lambda(\Phi^{t}), \nabla h_\lambda(\Phi^{t}) + r_\lambda^t + b_\Phi^t\rangle + \frac{\eta^2 E^2L_h^\lambda}{2}\left\|\nabla h_\lambda(\Phi^{t}) + r_\lambda^t + b_\Phi^t\right\|_F^2 \\
            &\le h_\lambda(\Phi^{t}) -\frac{\eta E}{2} \left\|\nabla h_\lambda(\Phi^t)\right\|_F^2 + \frac{\eta E}{2}\left\|r_\lambda^t + b_\Phi^t\right\|_F^2
        \end{aligned}
    \end{equation}
    By Danskin's theorem,
    \begin{equation}
        \begin{aligned}
            \left\|r_\lambda^t\right\|_F  = \left\|\nabla_\Phi f_{\lambda}(\Phi^t,\theta^t) - \nabla_\Phi f_\lambda(\Phi^t, \theta_\lambda^\star(\Phi^t))\right\|_F \le L_{\Phi\theta} \left\|\theta^t - \theta_\lambda^\star(\Phi^t)\right\|_F = L_{\Phi\theta} \left\|e_t\right\|_F
        \end{aligned}
    \end{equation}
    in addition, by Lemma \ref{lm3}, when $\eta L_f^\lambda E \le \epsilon /(1+\epsilon)$ for $\epsilon = 1/(4B)$, we have:
    \begin{equation}
        \begin{aligned}
            \left\|b_\Phi^t\right\|_F^2 & \le \epsilon^2\left(M_\lambda^2 + B^2 \left\|\nabla f_\lambda (\Phi^t,\theta^t)\right\|_F^2\right) \\
            &\le \frac{M_\lambda^2}{16B^2} + \frac{1}{8}\left\|\nabla f_\lambda (\Phi^t,\theta^t) - \nabla f_\lambda (\Phi^t,\theta_\lambda^\star(\Phi^t))\right\|_F^2 + \frac{1}{8}\left\|\nabla f_\lambda (\Phi^t,\theta_\lambda^\star(\Phi^t))\right\|_F^2 \\
            &\le \frac{M_\lambda^2}{16B^2} + \frac{1}{8}\cdot\left(L_{\Phi\theta}^2+\left(L_{\theta\theta}+\lambda\right)^2\right)\left\|\theta^t - \theta_\lambda^\star(\Phi^t)\right\|_F^2 + \frac18\left\|\nabla h_\lambda(\Phi^t)\right\|_F^2 \\
            &= \frac{M_\lambda^2}{16B^2} + \frac18\left(L_{\Phi\theta}^2+\left(L_{\theta\theta}+\lambda\right)^2\right)\left\|e_t\right\|_F^2 + \frac18\left\|\nabla h_\lambda(\Phi^t)\right\|_F^2
        \end{aligned}
    \end{equation}
    then,
    \begin{equation}
        \begin{aligned}
            h_\lambda(\Phi^{t+1}) &\le h_\lambda(\Phi^{t})-\frac{\eta E}{2} \left\|\nabla h_\lambda(\Phi^t)\right\|_F^2 + \frac{\eta E}{2}\left\|r_\lambda^t + b_\Phi^t\right\|_F^2 \\
            &\le h_\lambda(\Phi^{t})-\frac{\eta E}{2} \left\|\nabla h_\lambda(\Phi^t)\right\|_F^2 +\eta E\left(\left\|r_\lambda^t\right\|_F^2 + \left\|b_\Phi^t\right\|_F^2\right) \\
            &\le h_\lambda(\Phi^{t})-\frac{\eta E}{2} \left\|\nabla h_\lambda(\Phi^t)\right\|_F^2 +\eta E L_{\Phi\theta}^2 \left\|e_t\right\|_F^2 \\
            & \quad + \eta E\left(\frac{M_\lambda^2}{16B^2} + \frac18\left(L_{\Phi\theta}^2+\left(L_{\theta\theta}+\lambda\right)^2\right)\left\|e_t\right\|_F^2 + \frac18\left\|\nabla h_\lambda(\Phi^t)\right\|_F^2\right) \\
            &\le h_\lambda(\Phi^{t})-\frac{3\eta E}{8} \left\|\nabla h_\lambda(\Phi^t)\right\|_F^2 + \eta E(L_f^\lambda)^2\left\|e_t\right\|_F^2 + \frac{M_\lambda^2\eta E}{16B^2}
        \end{aligned}
    \end{equation}
    Based on the above, let $C =  2B_\theta C_1G_O$ and $C_e = (L_f^\lambda)^2$. We also note that $M_\lambda^2 = M^2 +C(B^2-1)\lambda$. When $\eta \le \min \left\{\lambda/((L_{\Phi\Phi}\lambda + L_{\Phi\theta}L_{\theta\Phi})E), 1/(\sqrt{C_e}(4B+1)E)\right\}$, the first part of Theorem~\ref{thm1} holds. Since $B_{t}^{\lambda} = h_\lambda(\Phi^t)-f_\lambda^\star$, we have
    \begin{equation}
        \begin{aligned}
            B_{t}^{\lambda} &\le  B_{t-1}^{\lambda} - \frac{3\eta E}{8}\left\|\nabla h_\lambda(\Phi^{t-1})\right\|_F^2 + C_e\eta E\left\|e_{t-1}\right\|_F^2 + \frac{M_\lambda^2\eta E}{16B^2} \\
            &\le B_{0}^{\lambda} - \frac{3\eta E}{8} \sum_{s=0}^{t-1}
            \norm{\nabla h_\lambda(\Phi^s)}_F^2 + C_e\eta E \sum_{s=0}^{t-1} \norm{e_s}_F^2
            + \frac{(M^2+C(B^2-1)\lambda)\eta E}{16B^2} t
        \end{aligned}
    \end{equation}
    which completes the proof of Theorem~\ref{thm1}.

\end{proof}

\subsection{Proof of Corollary~\ref{col1}}

Since
\begin{equation}
\begin{aligned}
h_\lambda(\Phi) - h_0(\Phi) &= \min_\theta \left( f(\Phi, \theta) + \frac{\lambda}{2} \|\theta\|_F^2 \right) - \min_\theta f(\Phi, \theta) \\
&\leq f(\Phi, \theta^\star(\Phi)) + \frac{\lambda}{2} \|\theta^\star(\Phi)\|_F^2 - h(\Phi) \\
&= \frac{\lambda}{2} \|\theta^\star(\Phi)\|_F^2.
\end{aligned}
\end{equation}
note that $\norm{\theta^\star(\Phi^0)}_F \le B_\theta$, then
\begin{equation}
\begin{aligned}
    B_t - B_0 &= h(\Phi^t) - h(\Phi^0)  \\
        & \le  h_\lambda(\Phi^t) - h_\lambda(\Phi^0) + \frac\lambda2 \|\theta^\star(\Phi^0)\|_F^2 \\
        & \le h_\lambda(\Phi^t) - h_\lambda(\Phi^0) + \frac\lambda2 B_\theta^2
\end{aligned}
\end{equation}
According to Theorem~\ref{thm1}, we have
\begin{equation}
    B_t \le B_{0} - \frac{3\eta E}{8} \sum_{s=0}^{t-1} \norm{\nabla h_\lambda(\Phi^s)}_F^2 + C_e\eta E \sum_{s=0}^{t-1} \norm{e_s}_F^2 + \frac{(M^2+C(B^2-1)\lambda)\eta E}{16B^2} t  + \frac\lambda2 B_\theta^2,
\end{equation}
which completes the proof of Corollary~\ref{col1}.

\subsection{Proof of Theorem~\ref{thm:restricted_descent_relative}}
\begin{lem}
\label{g_property}
    Assume Assumption~\ref{asmp:attention_assumption} holds, $g_{\PhiFrozen}^\lambda(\theta)$ is $\lambda$-strongly convex and $L_g^\lambda$-smooth, where $L_g^\lambda = C_1^2 L_O + \lambda$.
\end{lem}
\begin{proof}
    By definition, we have
    \begin{equation}
        g_{\PhiFrozen}^\lambda(\theta) = \sum_{k=1}^K p_k \E_{\mathcal D_k} \Bigl[
        \ell\bigl(O(H;\PhiFrozen,\theta),y\bigr) \Bigr] + \frac{\lambda}{2}\norm{\theta}_F^2,
    \end{equation}

    According to Assumption~\ref{asmp:attention_assumption}, the loss function $\ell(O,y)$ is convex with respect to $O$, and $O$ is a linear function of $\theta$, so that $\ell$ is convex with respect to $\theta$, and $g_{\PhiFrozen}^{\lambda}(\theta)$ is $\lambda$-strongly convex. For smoothness, according to Lemma~\ref{lm0}, $\left\| Z(H; \Phi)\right\|_F \le C_1$, then we have
    \begin{equation}
    \begin{aligned}
         & \quad \norm{\nabla_\theta g_{\PhiFrozen}^\lambda (\theta)-\nabla_\theta g_{\PhiFrozen}^\lambda(\theta')}_F \\
        &= \norm{\sum_{k=1}^K p_k \E_{\mathcal D_k} \Bigl[ Z(H; \PhiFrozen)^\top \nabla_O \ell(Z(H; \PhiFrozen)\theta,y) +\lambda\theta - Z(H; \PhiFrozen)^\top \nabla_O \ell(Z(H; \PhiFrozen)\theta',y) - \lambda\theta'\Bigr]  }_F \\
        &\le  \sum_{k=1}^K p_k \E_{\mathcal D_k} \Bigl[\norm{Z(H; \PhiFrozen)}_F
        \norm{\nabla_O \ell(Z(H; \PhiFrozen)\theta,y)-\nabla_O \ell(Z(H; \PhiFrozen)\theta',y)}_F\Bigr] + \lambda \norm{\theta-\theta'}_F \\
        &\le C_1^2 L_O \norm{\theta - \theta'}_F +  \lambda \norm{\theta-\theta'}_F \\
        &= (C_1^2 L_O + \lambda) \norm{\theta-\theta'}_F
    \end{aligned}
    \end{equation}
    which yields the smoothness of $g_{\PhiFrozen}^{\lambda}(\theta)$.
\end{proof}

\begin{lem}
\label{lem:gMB}
    Assume Assumption~\ref{asmp:attention_assumption} holds, for the fixed frozen kernel $\PhiFrozen$, we have
    \begin{equation}
    \sum_{k=1}^K p_k
    \norm{\nabla g_{\PhiFrozen,k}^{\lambda}(\theta)}_F^2
    \le (M_{g}^\lambda)^{2} + (B_g^\lambda)^2
    \norm{\nabla g_{\PhiFrozen}^\lambda(\theta)}_F^2,
\end{equation}
where $(B_g^\lambda)^2 = 2(L_{\theta\theta}+\lambda) /\lambda$ and $(M_{g}^\lambda)^{2} =  2\sum_{k=1}^K p_k \norm{\nabla g_{\PhiFrozen,k}^{\lambda}(\theta_\lambda^\star(\PhiFrozen))}_F^2$. Furthermore, we have $(M_{g}^\lambda)^{2} \le 8(C_1G_O)^2$.
\end{lem}

\begin{proof}

    Note that $g_{\PhiFrozen,k}^{\lambda}$ and $g_{\PhiFrozen}^{\lambda}$ are $(L_{\theta\theta}+\lambda)$-smooth and $\lambda$-strongly convex, we have:
    \begin{equation}
        \begin{aligned}
            \sum_{k=1}^K p_k \norm{\nabla g_{\PhiFrozen,k}^{\lambda}(\theta)}_F^2 &\le 2\sum_{k=1}^K p_k \left(\norm{\nabla g_{\PhiFrozen,k}^{\lambda}(\theta_\lambda^\star(\PhiFrozen))}_F^2+\norm{\nabla g_{\PhiFrozen,k}^{\lambda}(\theta)- \nabla g_{\PhiFrozen,k}^{\lambda}(\theta_\lambda^\star(\PhiFrozen))}_F^2 \right) \\
            &= 2\sum_{k=1}^K p_k \norm{\nabla g_{\PhiFrozen,k}^{\lambda}(\theta_\lambda^\star(\PhiFrozen))}_F^2 + 2\sum_{k=1}^K p_k \norm{\nabla g_{\PhiFrozen,k}^{\lambda}(\theta)- \nabla g_{\PhiFrozen,k}^{\lambda}(\theta_\lambda^\star(\PhiFrozen))}_F^2 \\
            &\le 2\sum_{k=1}^K p_k \norm{\nabla g_{\PhiFrozen,k}^{\lambda}(\theta_\lambda^\star(\PhiFrozen))}_F^2 \\ & \quad +     4\left(L_{\theta\theta}+\lambda\right)\sum_{k=1}^K p_k \left(g_{\PhiFrozen,k}^{\lambda}(\theta)-g_{\PhiFrozen,k}^{\lambda}(\theta_\lambda^\star(\PhiFrozen)) - \langle\nabla g_{\PhiFrozen,k}^{\lambda}(\theta_\lambda^\star(\PhiFrozen)), \theta - \theta_\lambda^\star(\PhiFrozen)\rangle\right) \\
            &= 2\sum_{k=1}^K p_k \norm{\nabla g_{\PhiFrozen,k}^{\lambda}(\theta_\lambda^\star(\PhiFrozen))}_F^2 + 4\left(L_{\theta\theta}+\lambda\right)\left(g_{\PhiFrozen}^{\lambda}(\theta)-g_{\PhiFrozen}^{\lambda}(\theta_\lambda^\star(\PhiFrozen))\right) \\
            &\le 2\sum_{k=1}^K p_k \norm{\nabla g_{\PhiFrozen,k}^{\lambda}(\theta_\lambda^\star(\PhiFrozen))}_F^2 + 4\left(L_{\theta\theta}+\lambda\right)\cdot \frac{1}{2\lambda} \norm{\nabla g_{\PhiFrozen}^{\lambda}(\theta)- \nabla g_{\PhiFrozen}^{\lambda}(\theta_\lambda^\star(\PhiFrozen))}_F^2 \\
            &=2\sum_{k=1}^K p_k \norm{\nabla g_{\PhiFrozen,k}^{\lambda}(\theta_\lambda^\star(\PhiFrozen))}_F^2 + \frac{2(L_{\theta\theta}+\lambda)}{\lambda}\norm{\nabla g_{\PhiFrozen}^{\lambda}(\theta)}_F^2
        \end{aligned}
    \end{equation}
    So that Lemma~\ref{lem:gMB} holds by letting
    \begin{equation}
    \begin{aligned}
        (B_g^\lambda)^2 &=   2(L_{\theta\theta}+\lambda) /\lambda  \\
        (M_{g}^\lambda)^{2} &=  2\sum_{k=1}^K p_k \norm{\nabla g_{\PhiFrozen,k}^{\lambda}(\theta_\lambda^\star(\PhiFrozen))}_F^2.
    \end{aligned}
    \end{equation}
    Furthermore, according to Assumption~\ref{asmp:attention_assumption} and Lemma~\ref{lm0}, we have
    \begin{equation}
        \begin{aligned}
            \norm{\nabla g_{\PhiFrozen,k}^{\lambda}(\theta_\lambda^\star(\PhiFrozen))}_F &= \norm{\nabla_\theta F_{k}(\PhiFrozen, \theta_\lambda^\star(\PhiFrozen))  +\lambda\theta_\lambda^\star(\PhiFrozen)}_F \\
            &=\norm{\nabla_\theta F_{k}(\PhiFrozen, \theta_\lambda^\star(\PhiFrozen))  -\nabla_\theta f(\PhiFrozen, \theta_\lambda^\star(\PhiFrozen))}_F \\
            &\le 2C_1G_O
        \end{aligned}
    \end{equation}
    then we have
    \begin{equation}
        \begin{aligned}
            (M_{g}^\lambda)^{2} \le 2\sum_{k =1 }^K p_k  (2C_1G_O)^2   =8(C_1G_O)^2
        \end{aligned}
    \end{equation}
\end{proof}

\begin{lem}
\label{bthetas}
    Assume Assumption~\ref{asmp:attention_assumption} holds, for the fixed kernel $\PhiFrozen$, when $L_{g}^{\lambda}\eta E\le \epsilon/(1+\epsilon)$ for any $\epsilon >0$, let the residual semantic drift at time $s$ be:
    \begin{equation}
        b_{\theta,\lambda}^s :=
    \frac1E\sum_{k=1}^K p_k \sum_{j=0}^{E-1}
    \Bigl(
        \nabla g_{\PhiFrozen,k}^{\lambda}(\theta_{k,j}^s)
        -
        \nabla g_{\PhiFrozen,k}^{\lambda}(\theta^s)
    \Bigr),
    \end{equation}
    then we have
    \begin{equation}
        \norm{b_{\theta, \lambda}^s}_F \le \epsilon\sqrt{(M_{g}^\lambda)^{2} + (B_g^\lambda)^2
    \norm{\nabla g_{\PhiFrozen}^\lambda(\theta^s)}_F^2}
    \end{equation}
\end{lem}

\begin{proof}
    Let $d_{k,j}^s = \theta_{k,j}^s-\theta^s$, based on Algorithm~\ref{alg:fedfrozen}, we have
    \begin{equation}
    d_{k,j+1}^s = d_{k,j}^s -
    \eta\bigl(\nabla g_{\PhiFrozen,k}^\lambda(\theta_{k,j}^s)-\nabla g_{\PhiFrozen,k}^\lambda(\theta^s)\bigr) -
    \eta \nabla g_{\PhiFrozen,k}^\lambda(\theta^s).
    \end{equation}
    According to Lemma~\ref{g_property}, $g_{\PhiFrozen,k}^\lambda$ is $L_g^\lambda$ smooth, then we have
    \begin{equation}
    \norm{d_{k,j+1}^s}_F \le (1+L_g^\lambda \eta)\norm{d_{k,j}^s}_F
    + \eta \norm{\nabla g_{\PhiFrozen,k}^\lambda(\theta^s)}_F.
    \end{equation}
    Since $d_{k,0}^s=0$,  we have
    \begin{equation}
    \norm{d_{k,j}^s}_F \le \frac{(1+L_g^\lambda \eta)^j-1}{L_g^\lambda}
    \norm{\nabla g_{\PhiFrozen,k}^\lambda(\theta^s)}_F.
    \end{equation}
    When $L_{g}^{\lambda}\eta E\le \epsilon/(1+\epsilon)$,  we  have $(1+ \eta L_g^\lambda)^E -1 \le (1+ \epsilon)\eta L_g^\lambda E$, then,
    \begin{equation}
        \begin{aligned}
            \norm{b_{\theta, \lambda}^s}_F &\le \frac1E \sum_{k=1}^K p_k \sum_{j=0}^{E-1} \norm{\nabla g_{\PhiFrozen,k}^\lambda(\theta_{k,j}^s)-\nabla g_{\PhiFrozen,k}^\lambda(\theta^s)}_F  \\
            &\le \frac{L_g^\lambda}{E} \sum_{k=1}^K p_k \sum_{j=0}^{E-1} \norm{d_{k,j}^s}_F  \\
            &\le \frac{L_g^\lambda}E \sum_{k=1}^K p_k \sum_{j=0}^{E-1} \frac{(1+L_g^\lambda \eta)^j-1}{L_g^\lambda} \norm{\nabla g_{\PhiFrozen,k}^\lambda(\theta^s)}_F \\
            &=  \left(\frac{(1+L_g^\lambda \eta)^E-1}{L_g^\lambda \eta E} - 1 \right) \sum_{k=1}^K p_k \left\|\nabla g_{\PhiFrozen,k}^\lambda(\theta^s)\right\|_F \\
            &\le \epsilon  \sum_{k=1}^K p_k \left\|\nabla g_{\PhiFrozen,k}^\lambda(\theta^s)\right\|_F
        \end{aligned}
    \end{equation}
    According to Cauchy-Schwarz Inequality, we have
    \begin{equation}
        \norm{b_{\theta, \lambda}^s}_F^2 \le \left( \epsilon  \sum_{k=1}^K p_k \left\|\nabla g_{\PhiFrozen,k}^\lambda(\theta^s)\right\|_F \right)^2 \le \epsilon ^2  \sum_{k=1}^K p_k \left\|\nabla g_{\PhiFrozen,k}^\lambda(\theta^s)\right\|_F^2
    \end{equation}
    Applying Lemma~\ref{lem:gMB}, we have
    \begin{equation}
        \norm{b_{\theta, \lambda}^s}_F \le \epsilon  \sqrt{(M_{g}^\lambda)^{2} + (B_g^\lambda)^2
    \norm{\nabla g_{\PhiFrozen}^\lambda(\theta^s)}_F^2}
    \end{equation}
\end{proof}

Now we are ready to prove Theorem~\ref{thm:restricted_descent_relative}.
\begin{proof}
    Since
    \begin{equation}
        \theta^{s+1}-\theta^s = -\eta E\bigl(\nabla g_{\PhiFrozen}^\lambda(\theta^s)+b_{\theta,\lambda}^s\bigr).
    \end{equation}
    According to Lemma~\ref{g_property}, $g_{\PhiFrozen}^\lambda(\theta)$ is $L_g^\lambda$ smooth, when $\eta \le 1/(L_g^\lambda E)$, we have
    \begin{equation}
        \begin{aligned}
            g_{\PhiFrozen}^\lambda(\theta^{s+1}) &\le g_{\PhiFrozen}^\lambda(\theta^s) +  \inner{\nabla g_{\PhiFrozen}^\lambda(\theta^s)}{\theta^{s+1}-\theta^s} +  \frac{L_g^\lambda}{2}\norm{\theta^{s+1}-\theta^s}_F^2 \\
            &= g_{\PhiFrozen}^\lambda(\theta^s) - \eta E \norm{\nabla g_{\PhiFrozen}^\lambda(\theta^s)}_F^2 - \eta E\inner{\nabla g_{\PhiFrozen}^\lambda(\theta^s)}{b_{\theta,\lambda}^s} + \frac{L_g^\lambda \eta^2E^2}{2}\norm{ \nabla g_{\PhiFrozen}^\lambda(\theta^s)+b_{\theta,\lambda}^s}_F^2 \\
            &\le g_{\PhiFrozen}^\lambda(\theta^s) - \eta E \norm{\nabla g_{\PhiFrozen}^\lambda(\theta^s)}_F^2 - \eta E\inner{\nabla g_{\PhiFrozen}^\lambda(\theta^s)}{b_{\theta,\lambda}^s} + \frac{\eta E}{2}\norm{ \nabla g_{\PhiFrozen}^\lambda(\theta^s)+b_{\theta,\lambda}^s}_F^2 \\
            &= g_{\PhiFrozen}^\lambda(\theta^s) - \frac{\eta E}2 \norm{\nabla g_{\PhiFrozen}^\lambda(\theta^s)}_F^2  +\frac{\eta E}{2}\norm{b_{\theta,\lambda}^s}_F^2.
        \end{aligned}
    \end{equation}

    When $L_{g}^{\lambda}\eta E\le \epsilon /(1+ \epsilon)$, by Lemma~\ref{bthetas},
    \begin{equation}
        \norm{b_{\theta, \lambda}^s}_F \le \epsilon  \sqrt{(M_{g}^\lambda)^{2} + (B_g^\lambda)^2
    \norm{\nabla g_{\PhiFrozen}^\lambda(\theta^s)}_F^2},
    \end{equation}
    By letting $\epsilon = 1/ 2 B_g^\lambda$, according to Lemma~\ref{lem:gMB}, we have
    \begin{equation}
        \begin{aligned}
            g_{\PhiFrozen}^\lambda(\theta^{s+1}) &\le g_{\PhiFrozen}^\lambda(\theta^{s}) - \frac{\eta E}{2} \norm{\nabla g_{\PhiFrozen}^\lambda(\theta^s) }_F^2 + \frac
            {\eta E}2\cdot \epsilon^2  \left( (M_{g}^\lambda)^{2} + (B_g^\lambda)^2 \norm{\nabla g_{\PhiFrozen}^\lambda(\theta^s)}_F^2 \right) \\
            &= g_{\PhiFrozen}^\lambda(\theta^{s}) - \frac{3\eta E}8  \norm{\nabla g_{\PhiFrozen}^\lambda(\theta^s) }_F^2 + \frac{\eta E (M_{g}^\lambda)^{2}}{8(B_g^\lambda)^2} \\
            &\le g_{\PhiFrozen}^\lambda(\theta^{s}) - \frac{3\eta E}8  \norm{\nabla g_{\PhiFrozen}^\lambda(\theta^s) }_F^2 +\frac{\eta E \lambda (C_1G_O)^{2}}{2(L_{\theta\theta}+ \lambda)} \\
            &\le g_{\PhiFrozen}^\lambda(\theta^{s}) - \frac{3\eta E}8  \norm{\nabla g_{\PhiFrozen}^\lambda(\theta^s) }_F^2 +\frac{\eta E (C_1G_O)^{2}}{2}
        \end{aligned}
    \end{equation}
    Let $C = C_1G_O$, then we have
    \begin{equation}
         g_{\PhiFrozen}^\lambda(\theta^{s+1}) \le g_{\PhiFrozen}^\lambda(\theta^{s}) - \frac{3}{8}  \eta E \norm{\nabla g_{\PhiFrozen}^\lambda(\theta^s) }_F^2 + \frac{C^2 \eta E}{2}
    \end{equation}

    Summing over $s=0,\dots,S-1$ and using $g_{\PhiFrozen}^\lambda(\theta^S)\ge h_\lambda(\PhiFrozen)$ yields
    \begin{equation}
        \frac{1}{S}\sum_{s=0}^{S-1}\norm{\nabla g_{\PhiFrozen}^\lambda (\theta^s)}_F^2 \le \frac{8(g_{\PhiFrozen}^\lambda(\theta^{0}) - h_\lambda(\PhiFrozen))}{3 \eta E S} + \frac{4C^2}{3}
    \end{equation}
    Based on the above, we let $c= \bigl((C_1^2L_O+\lambda)(2\sqrt{2(L_{\theta\theta}+\lambda)}+\sqrt{\lambda})\bigr)^{-1}$, then Theorem~\ref{thm:restricted_descent_relative} holds when $\eta \le c \sqrt{\lambda}/E$. Notably, $c$ is treated as a small constant because if we additionally have $\lambda\le B_\lambda$, it suffices to take $c = \bigl((C_1^2L_O+B_\lambda)(2\sqrt{2(L_{\theta\theta}+ B_\lambda)}+ B_\lambda)\bigr)^{-1}$.
\end{proof}

\subsection{Proof of Corollary~\ref{col:restricted_linear}}
\begin{proof}
    According to Lemma~\ref{g_property}, $g_{\PhiFrozen}^\lambda(\theta)$ is $\lambda$-strongly convex, then we have
    \begin{equation}
        \begin{aligned}
            \norm{\nabla g_{\PhiFrozen}^\lambda(\theta)}_F^2 \ge 2 \lambda (g_{\PhiFrozen}^\lambda(\theta) - h_\lambda(\PhiFrozen))
        \end{aligned}
    \end{equation}
    Substituting this into Theorem~\ref{thm:restricted_descent_relative} yields
    \begin{equation}
        \begin{aligned}
            g_{\PhiFrozen}^\lambda(\theta^{s+1})  -h_\lambda (\PhiFrozen) &\le g_{\PhiFrozen}^\lambda(\theta^{s})  -h_\lambda(\PhiFrozen) - \frac{3}{8}  \eta E \norm{\nabla g_{\PhiFrozen}^\lambda(\theta^s) }_F^2 + \frac{C'^2 \eta E}{2}\\
            &\le g_{\PhiFrozen}^\lambda(\theta^{s})- h_\lambda(\PhiFrozen) -\frac{3}{8}\eta E \cdot 2\lambda\left(g_{\PhiFrozen}^\lambda(\theta^s)- h_\lambda(\PhiFrozen)\right)+ \frac{C'^2 \eta E}{2}  \\
            &=\left(1-\frac{3}{4}\lambda\eta E\right)\left(g_{\PhiFrozen}^\lambda(\theta^{s})-h_\lambda(\PhiFrozen)\right)+ \frac{C'^2 \eta E}{2}.
        \end{aligned}
    \end{equation}
    Let $1-\rho_\lambda=\frac34\lambda\eta E$. When $0< 3\lambda \eta  E /4 <1 $, we have
    \begin{equation}
    \label{prop_proof}
    \begin{aligned}
        g_{\PhiFrozen}^\lambda(\theta^S) - h_\lambda(\PhiFrozen)
        &\le \left(1-\frac{3}{4}\lambda\eta E \right)^S \left( g_{\PhiFrozen}^\lambda(\theta^0) - h_\lambda(\PhiFrozen) \right)
        + \frac{C'^2 \eta E}{2(1-\rho_\lambda)}
        (1-\rho_\lambda^S) \\
        &\le \left(1-\frac{3}{4}\lambda\eta E \right)^S \left( g_{\PhiFrozen}^\lambda(\theta^0) - h_\lambda(\PhiFrozen) \right) +
        \frac{2C'^2}{3\lambda}.
    \end{aligned}
    \end{equation}
    According to the definition, $g_{\PhiFrozen}(\theta) = g_{\PhiFrozen}^0 (\theta)$ and $h(\Phi) = h_0(\Phi)$, and note that $\PhiFrozen = \Phi^{T_{\rm warm}}$, then
    \begin{equation}
        \begin{aligned}
            g_{\PhiFrozen}(\theta^S) - h(\PhiFrozen) &\le  g_{\PhiFrozen}^\lambda(\theta^S) - h_\lambda(\PhiFrozen) +\frac\lambda 2\norm{\theta^\star(\PhiFrozen)}_F^2\\
            &\le  g_{\PhiFrozen}^\lambda(\theta^S) - h_\lambda(\PhiFrozen) + \frac{\lambda B_\theta^2}{2} \\
            &\le \left(1-\frac{3}{4}\lambda\eta E \right)^S \left( g_{\PhiFrozen}^\lambda(\theta^0) - h_\lambda(\PhiFrozen) \right) +
        \frac{2C'^2}{3\lambda} + \frac{\lambda B_\theta^2}{2},
        \end{aligned}
    \end{equation}
    Let $C = 7C'/6$ and $\lambda_0 = (6C)/(7B_\theta)$. Then we have
    \begin{equation}
        \begin{aligned}
             g_{\PhiFrozen}(\theta^S) - h(\PhiFrozen) &\le \left(1-\frac{3}{4}\lambda_0 \eta E \right)^S \left( g_{\PhiFrozen}^{\lambda_0}(\theta^0) - h_{\lambda_0}(\PhiFrozen) \right) + C B_\theta
        \end{aligned}
    \end{equation}
    so that Corollary~\ref{col:restricted_linear} holds when $\eta \le \min \{c \sqrt{\lambda_0}/E , 4 /(3 \lambda_0 E) \}$ for constant $c$ defined in Theorem~\ref{thm:restricted_descent_relative}.
\end{proof}

\subsection{Proof of Proposition~\ref{prop:e2e_certificate}}

\begin{proof}
Fix a candidate freezing round $\tau\in\{0,\dots,T-1\}$ and set
$S=T-\tau$.  According to the condition,
\begin{equation}
    f_\lambda^\star \le f^\star+ \frac\lambda 2 \norm{\theta^\star}_F^2\le f^\star+\frac{\lambda}{2}B_\theta^2,
\end{equation}
then we have
\begin{equation}
    f(w_{\tau,T})-f^\star \le f_\lambda(w_{\tau,T})-f_\lambda^\star + \frac{\lambda}{2}B_\theta^2 .
\end{equation}
Since the kernel is fixed at $\Phi^\tau$ during the  Phase~2 procedure, then we have
\begin{equation}
    f_\lambda(w_{\tau,T}) = g_{\Phi^\tau}^\lambda(\theta_{\tau,S}).
\end{equation}
Hence
\begin{equation}
    \begin{aligned}
        f_\lambda(w_{\tau,T})-f_\lambda^\star
    &= \Bigl( g_{\Phi^\tau}^\lambda(\theta_{\tau,S}) -
        h_\lambda(\Phi^\tau) \Bigr) + \Bigl(
        h_\lambda(\Phi^\tau)-f_\lambda^\star \Bigr)  \\
    &= \Bigl( g_{\Phi^\tau}^\lambda(\theta_{\tau,S}) -
        h_\lambda(\Phi^\tau) \Bigr) + B_\tau^\lambda .
    \end{aligned}
\end{equation}
According to the proof of Corollary~\ref{col:restricted_linear}, we have
\begin{equation}
    f(w_{\tau,T})-f^\star \le B_\tau^{\lambda_0} +
    \rho_{\lambda_0}^{T-\tau}R_\tau^{\lambda_0} + C B_\theta,
\end{equation}
where $\lambda_0 = (6C)/(7B_\theta)$ and $C>0$ is a positive constant defined in Corollary~\ref{col:restricted_linear}.
\end{proof}

\subsection{Remarks of Assumption~\ref{asmp:attention_assumption}}
\begin{remark}
    Assumption~\ref{asmp:4} is satisfied by several commonly used smooth and
bounded feature maps. In particular, for any row vector $u$, the row-wise softmax feature map
\begin{equation}
    \varphi(u)_i = \frac{\exp(u_i)}{\sum_{j=1}^m \exp(u_j)}.
\end{equation}
is bounded and has a bounded Lipschitz-continuous Jacobian.
\end{remark}

\begin{remark}
Assumption~\ref{asmp:5} is satisfied by commonly used loss functions after a fixed
linear prediction map is applied to the matrix-valued attention output $O$.

For regression, the mean squared error loss with a fixed prediction matrix $W$
\begin{equation}
    \ell(O,y)
    =
    \frac12\bigl(\langle O,W\rangle-y\bigr)^2;
\end{equation}
For classification,  the softmax cross-entropy loss with a fixed
linear prediction vector $V$
\begin{equation}
    \ell(O,y)
    =
    -\sum_{i=1}^m y_i\log \operatorname{softmax}(OV)_i.
\end{equation}
\end{remark}

\section{Simulation Settings and Additional Results}
\subsection{Simulation Settings}
All simulation experiments use the federated linear-attention setup described in the main text. Unless otherwise stated, the system consists of $K=10$ clients, model dimensions $d=64$ and $d_k=d_v=32$, $n=16$ tokens per example, and $40$ examples per client. Each client distribution is generated by drawing a client-specific mean $\mu_k\sim\mathcal{N}(0,\rho^2 I_d)$ and then sampling token embeddings $H\sim\mathcal{N}(\mu_k,I_d)$. Labels are produced by a shared single-layer linear-attention global model with additive Gaussian noise of standard deviation $10^{-2}$. The local models minimize the mean squared error with ridge regularization $\lambda=10^{-3}$ applied to the value block.

We use full client participation and uniform server averaging across all simulations. During each communication round, clients perform $E$ local epochs using full-batch SGD. The default training consists of $T=200$ communication rounds and $E=20$ local epochs. Reported curves summarize ten random seeds using medians, with interquartile ranges shaded.

\begin{table}[h]
\centering
\caption{Default simulation hyperparameters.}
\label{tab:simulation_hyperparameters}
\begin{tabular}{lc}
\toprule
Hyperparameter & Value \\
\midrule
Number of clients $K$ & $10$ \\
Embedding dimension $d$ & $64$ \\
Query/key dimension $d_k$ & $32$ \\
Value dimension $d_v$ & $32$ \\
Tokens per example $n$ & $16$ \\
Examples per client & $40$ \\
Noise standard deviation & $10^{-2}$ \\
Ridge coefficient $\lambda$ & $10^{-3}$ \\
Communication rounds $T$ & $200$ \\
Local epochs $E$ & $20$ \\
Client weights & Uniform \\
\addlinespace[1.5pt]
Seeds & $0,\ldots,9$ \\
\bottomrule
\end{tabular}
\end{table}

The phase-profile experiment in Figure~\ref{fig:sim_phase1_profile} uses a learning rate $\eta=10^{-3}$. The warm-up trade-off and local-epoch experiments use $\eta=2\times 10^{-4}$. For evaluating the sensitivity to $\rho$ in Figure~\ref{fig:sim_main_results}, baseline methods follow the exact same training configuration, with the proximal term for FedProx specifically set to $\mu=200$.

\subsection{Additional Simulation Results}

Figure~\ref{fig:app_supp_simulation} presents two additional simulation experiments to further validate our approach. First, we examine the stability of the warm-up trade-off across varying heterogeneity levels. We evaluate \modelname across $\rho\in\{0,0.5,1.0,1.5,2.0\}$ and warm-up ratios $T_{\rm warm}/T\in\{0,0.25,0.5,0.75,1.0\}$, while maintaining $T=200$, $E=20$ local epochs, and $\eta=2\times10^{-4}$. Each heatmap entry represents the median final loss over ten random seeds on a $\log_{10}$ scale, with stars indicating the optimal warm-up fraction for each heterogeneity level. Second, we analyze the sensitivity to the $\ell_2$ regularization penalty on the value block by testing $\lambda\in\{0,10^{-1},1,10,100\}$, reporting both the final loss and the average client drift over the final 50 communication rounds. The results are shown in Figure~\ref{fig:app_supp_simulation}.

\begin{figure}[h]
    \centering
    \begin{minipage}{0.48\textwidth}
        \centering
        \includegraphics[width=\linewidth]{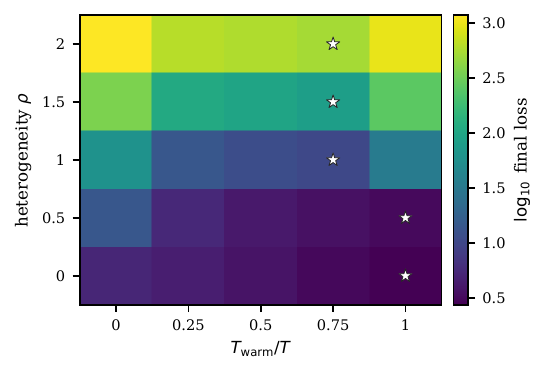}
        \centerline{(a) Warm-up robustness}
    \end{minipage}
    \hfill
    \begin{minipage}{0.48\textwidth}
        \centering
        \includegraphics[width=\linewidth]{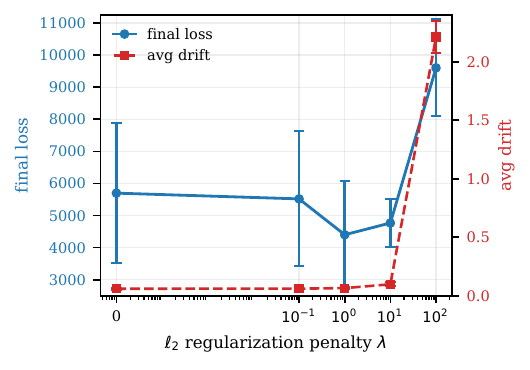}
        \centerline{(b) $\ell_2$ penalty effect}
    \end{minipage}
    \caption{
    (a) Warm-up robustness across
    heterogeneity levels $\rho$. (b) Impact of the value-block $\ell_2$ regularization penalty $\lambda$ on final loss and average client drift.
    }
    \label{fig:app_supp_simulation}
\end{figure}

Figure~\ref{fig:app_supp_simulation} presents two additional sets of simulation experiments to further validate our approach. By evaluating the median final loss across varying warm-up fractions ($T_{\rm warm}/T$), Panel~(a) corroborates the bias--drift interpretation discussed in the main text. Under low heterogeneity (small $\rho$), client drift is minimal, allowing a prolonged warm-up phase to continuously refine the shared attention kernel. However, as heterogeneity increases ($\rho\ge 1$), intermediate freezing strictly outperforms full-model training (i.e., $T_{\rm warm}/T=1.0$). This confirms that continuous updates to the query/key block under non-IID data eventually accumulate heterogeneity-induced drift. Furthermore, Panel~(b) illustrates the impact of the value-block $\ell_2$ regularization penalty $\lambda$ on both the final loss and average client drift. This reveals a complementary regularization trade-off: a moderate $\ell_2$ penalty effectively stabilizes the value-block optimization, whereas an excessively large penalty causes underfitting and subsequently amplifies client drift.

\section{Additional Real-Data Experimental Details}
\label{app:real_data_settings}

\subsection{Data partitioning.}
To ensure a fair evaluation, all methods are evaluated under identical conditions for each dataset, sharing the same random seed (fixed at $42$), data partitions, and model initialization. To simulate data heterogeneity, we construct class-wise Dirichlet non-IID partitions from the subsampled training sets. Specifically, samples from each class are randomly shuffled and distributed among the $K$ clients according to a $\mathrm{Dirichlet}(\alpha \mathbf{1}_K)$ distribution, yielding the local dataset $\mathcal{D}_k$ for each client $k$. We control the heterogeneity level via the concentration parameter $\alpha \in \{0.1, 0.3, 0.5\}$, where smaller values induce stronger data imbalance. Besides, we assume full client participation across all datasets, utilizing subsets of 5,000 training and 2,000 testing samples for each setting. Specifically, CIFAR-10 (10 classes) and FEMNIST (62 classes) are partitioned among 10 clients using local batch sizes of 32 and 64, respectively, while CIFAR-100 (100 classes) is distributed across 20 clients with a batch size of 32. To guarantee valid local training, any partition assigning fewer than the minimum required samples to a client is discarded and resampled.

\subsection{Model architectures and preprocessing.}
Table~\ref{tab:real_data_preprocessing} details the selected Vision Transformer (ViT) backbones, input resolutions, and dataset-specific preprocessing steps. Notably, to ensure robust feature extraction, all backbones are initialized with ImageNet-pretrained weights and equipped with tailored classification heads suited for their respective tasks.

\begin{table}[htbp!]
\centering
\scriptsize
\setlength{\tabcolsep}{4pt}
\caption{Model architectures and preprocessing pipelines. All backbones are initialized with ImageNet-pretrained weights and tailored classification heads.}
\label{tab:real_data_preprocessing}
\resizebox{\textwidth}{!}{%
\arrayrulecolor{black!25}
\begin{tabular}{@{}lccc p{0.5\textwidth}@{}}
\arrayrulecolor{black}
\toprule
Dataset & Backbone & Image size & Normalization & Preprocessing \\
\midrule
CIFAR-10
& ViT-B/32
& $224 \times 224$
& CIFAR-10 mean/std
& Resize to $224 \times 224$, normalize. \\
\arrayrulecolor{black!25}\hdashline\arrayrulecolor{black}
FEMNIST
& ViT-Small/16
& $224 \times 224$
& EMNIST mean/std
& Fix orientation, grayscale to RGB, resize to $224 \times 224$, normalize.  \\
\arrayrulecolor{black!25}\hdashline\arrayrulecolor{black}
\addlinespace[1.5pt]
CIFAR-100
& ViT-B/32
& $224 \times 224$
& CIFAR-100 mean/std
& Resize to $224 \times 224$, normalize.  \\
\arrayrulecolor{black}
\bottomrule
\end{tabular}%
}
\end{table}

\subsection{Baseline Hyperparameters}
To maintain a rigorous comparison, we standardize the training budget across all real-data experiments to $T=10$ global communication rounds, with each participating client performing $E=10$ local epochs per round. To facilitate stable convergence, the local learning rate follows a cosine annealing schedule applied over these communication rounds, decaying from the baseline-specific initial values detailed in Table~\ref{tab:real_data_method_hparams}. The global model is systematically evaluated on the central test set at the end of every round. Unless otherwise specified, local training utilizes the AdamW optimizer with a base learning rate of $10^{-4}$ and a weight decay of $0.01$. Table~\ref{tab:real_data_method_hparams} provides a comprehensive breakdown of these hyperparameter configurations, including algorithm-specific adjustments (e.g., server-side learning rates, proximal terms) and the dataset-specific warm-up schedules tailored for our proposed \modelname.

\begin{table}[htbp!]
\centering
\scriptsize
\setlength{\tabcolsep}{2.5pt}
\caption{Hyperparameter configurations for all federated learning methods evaluated on real-world datasets. The table details local optimizers, learning rates, and method-specific settings to ensure reproducible and fair comparisons.}
\label{tab:real_data_method_hparams}
\resizebox{\textwidth}{!}{%
\arrayrulecolor{black!25}
\begin{tabular}{@{}llcccl@{}}
\arrayrulecolor{black}
\toprule
Dataset & Method & Local optimizer & Local LR & Weight decay & Method-specific settings \\
\midrule
\multirow{7}{*}{CIFAR-10}
& FedAvg & AdamW & $10^{-4}$ & $0.01$ & Standard weighted averaging \\
& FedProx & AdamW & $10^{-4}$ & $0.01$ & $\mu=0.01$ \\
& SCAFFOLD & SGD & $0.1$ & $0$ & Momentum $0$; client/server control variates \\
& FedAvgM & AdamW & $10^{-4}$ & $0.01$ & Server LR $1.0$, server momentum $0.9$ \\
& FedAdam & AdamW & $10^{-4}$ & $0.01$ & Server LR $10^{-3}$, $\beta_1=0.9$, $\beta_2=0.99$, $\tau=10^{-3}$ \\
& FedNova & AdamW & $10^{-4}$ & $0.01$ & Normalized local update aggregation, server LR $1.0$ \\
& \modelname & AdamW & $10^{-4}$ & $0.01$ & Warm-up $5$ rounds; freeze Q/K from round $6$ \\
\arrayrulecolor{black!25}\hdashline\arrayrulecolor{black}
\multirow{7}{*}{FEMNIST}
& FedAvg & AdamW & $10^{-4}$ & $0.01$ & Standard weighted averaging \\
& FedProx & AdamW & $10^{-4}$ & $0.01$ & $\mu=0.01$ \\
& SCAFFOLD & SGD & $0.003$ & $0$ & Momentum $0.9$; client/server control variates \\
& FedAvgM & AdamW & $10^{-4}$ & $0.01$ & Server LR $1.0$, server momentum $0.9$ \\
& FedAdam & AdamW & $10^{-4}$ & $0.01$ & Server LR $10^{-2}$, $\beta_1=0.9$, $\beta_2=0.99$, $\tau=10^{-3}$ \\
& FedNova & AdamW & $10^{-4}$ & $0.01$ & Normalized local update aggregation, server LR $1.0$ \\
& \modelname & AdamW & $10^{-4}$ & $0.01$ & Warm-up rounds $(2,5,5)$ for $\alpha=(0.1,0.3,0.5)$ \\
\arrayrulecolor{black!25}\hdashline\arrayrulecolor{black}
\multirow{7}{*}{CIFAR-100}
& FedAvg & AdamW & $10^{-4}$ & $0.01$ & Standard weighted averaging \\
& FedProx & AdamW & $10^{-4}$ & $0.01$ & $\mu=0.01$ \\
& SCAFFOLD & SGD & $0.003$ & $0$ & Momentum $0.9$; client/server control variates \\
& FedAvgM & AdamW & $10^{-4}$ & $0.01$ & Server LR $1.0$, server momentum $0.9$ \\
& FedAdam & AdamW & $10^{-4}$ & $0.01$ & Server LR $10^{-3}$, $\beta_1=0.9$, $\beta_2=0.99$, $\tau=10^{-3}$ \\
& FedNova & AdamW & $10^{-4}$ & $0.01$ & Normalized local update aggregation, server LR $1.0$ \\
\addlinespace[1.5pt]
& \modelname & AdamW & $10^{-4}$ & $0.01$ & Warm-up $3$ rounds; freeze Q/K from round $4$ \\
\arrayrulecolor{black}
\bottomrule
\end{tabular}%
}
\end{table}

\section{Computing Resources}
All experiments were conducted on a high-performance server. GPU acceleration is provided by 8 NVIDIA H800 GPUs, each equipped with 80 GiB of VRAM. The system's CPU processing is handled by dual Intel Xeon Platinum 8480CL processors, delivering a total of 112 physical cores and 224 logical threads. Additionally, the setup includes 2.0 TiB of system memory to support large-scale federated data loads, with all deep learning operations executed using the CUDA 12.8 toolkit.

\section{Existing Asset Licenses}
\label{app:asset_licenses}
We use only publicly available benchmark datasets, model implementations, and pretrained weights. CIFAR-10 and CIFAR-100 are public research benchmark datasets released by the original authors and cited through the CIFAR technical report~\citep{krizhevsky2009learning}; the UCI listing for CIFAR-10 records a CC BY 4.0 license. FEMNIST is obtained through the LEAF benchmark suite~\citep{caldas2018leaf}, whose code is released under a BSD-2-Clause license; the underlying EMNIST data are publicly distributed by NIST for research use. Vision Transformer backbones and pretrained weights are accessed through standard model libraries: torchvision is released under a BSD-style license, and timm is released under Apache-2.0. We do not redistribute the raw datasets or ImageNet data; the supplementary code uses these assets through their standard public download and library interfaces.


\end{document}